\def\VersionLong{}
\def\VersionFinal{}

\ifdefined\VersionLong%
\newcommand{\LongVersion}[1]{#1}
\newcommand{\ShortVersion}[1]{}
\else
\newcommand{\LongVersion}[1]{}
\newcommand{\ShortVersion}[1]{#1}
\fi

\documentclass[a4paper,10pt,envcountsame]{llncs}

\usepackage[utf8]{inputenc}
\usepackage[english]{babel}

\usepackage[ruled,lined,linesnumbered,noend]{algorithm2e}
\SetKwInOut{Input}{input}
\SetKwInOut{Output}{output}
\SetKw{KwPush}{push}
\SetKw{KwPop}{pop}
\SetKw{KwMerge}{merge}
\SetKw{KwFrom}{from}
\SetKw{KwTo}{to}
\SetKw{KwCompute}{compute}
\SetAlgorithmName{Algorithm}{algorithm}{List of Algorithms}

\usepackage{subcaption}
\usepackage{paralist} 
\usepackage{xspace}

\newenvironment{ienumeration}
{\ifdefined\VersionLong\begin{enumerate}\else\begin{inparaenum}[\itshape i\upshape)]\fi} 
		{\ifdefined\VersionLong\end{enumerate}\else\end{inparaenum}\fi} 

\newenvironment{oneenumeration}
{\ifdefined\VersionLong\begin{enumerate}\else\begin{inparaenum}[1)]\fi} 
		{\ifdefined\VersionLong\end{enumerate}\else\end{inparaenum}\fi} 

\newenvironment{myitemize}
	{\ifdefined\VersionLong\begin{itemize}\else\begin{inparaitem}[]\fi}
	{\ifdefined\VersionLong\end{itemize}\else\end{inparaitem}\fi} 

\usepackage{amsmath} 
\usepackage{amssymb} 
\usepackage{stmaryrd} 
\usepackage{mathabx} 
\usepackage{multirow}
\usepackage{mathtools}
\usepackage{mathrsfs}

\usepackage[misc,geometry]{ifsym} 

\usepackage{eurosym} 

\usepackage{listings}
\usepackage{booktabs}
\usepackage{siunitx}

\ifdefined\VersionWithComments%
\usepackage{draftwatermark}
\SetWatermarkText{draft}
\SetWatermarkScale{15}
\SetWatermarkColor[gray]{0.9}
\fi

\usepackage[svgnames,table]{xcolor}
\usepackage{colortbl}

\definecolor{mygreen}{rgb}{0,0.6,0}
\definecolor{mygray}{rgb}{0.5,0.5,0.5}
\definecolor{mymauve}{rgb}{0.58,0,0.82}
\definecolor{weborange}{RGB}{255,165,0}

\definecolor{darkblue}{rgb}{0, 0, 0.7}

\usepackage[
pdfauthor={Masaki Waga, Ezequiel Castellano, Sasinee Pruekprasert, Stefan Klikovits, Toru Takisaka, and Ichiro Hasuo},
pdftitle={Dynamic Shielding for Reinforcement Learning in Black-Box Environments},
breaklinks  = true,
colorlinks  = true,
citecolor   = blue!50!blue,
linkcolor   = darkblue,
urlcolor    = blue!50!green,
]{hyperref}
\usepackage[capitalise,english,nameinlink]{cleveref} 
\crefname{line}{\text{line}}{\text{lines}} 
\crefname{example}{\text{Example}}{\text{Examples}} 
\crefname{assumption}{\text{Assumption}}{\text{Assumptions}} 
\crefname{algorithm}{\text{Algorithm}}{\text{Algorithms}}

\usepackage{tikz}
\usetikzlibrary{arrows,automata,positioning,arrows.meta}
\usetikzlibrary{calc}
\tikzstyle{every node}=[initial text=]
\tikzstyle{location}=[circle, minimum size=12pt, draw=black, fill=blue!10, inner sep=1pt] 
\tikzstyle{invariant}=[draw=black, dotted, inner sep=1pt] 
\tikzstyle{final}=[double distance=3pt]
\tikzstyle{accepting}=[final]
\usepackage{pgfplots}
\pgfplotsset{every axis legend/.style={
}}
\pgfplotsset{every axis/.append style={
                    label style={font=\scriptsize},
                    tick label style={font=\scriptsize}
                    }}
\DeclareUnicodeCharacter{2212}{−}
\usepgfplotslibrary{groupplots,dateplot}
\usetikzlibrary{patterns,shapes.arrows}
\pgfplotsset{compat=newest}
\pgfplotsset{scaled x ticks=false}
\pgfplotsset{scaled y ticks=false}

\ifdefined\VersionWithComments%

\newcommand{\stylebenchmark}[1]{\textcolor{red!30!black}{\textsc{#1}}}

\else

\newcommand{\stylebenchmark}[1]{\textsc{#1}}

\fi

\ifdefined\VersionWithComments%
\definecolor{coloract}{rgb}{0.50, 0.70, 0.30}
\definecolor{colorclock}{rgb}{0.4, 0.4, 1}
\definecolor{colorconst}{rgb}{0.50, 0.20, 0.00}
\definecolor{colordisc}{rgb}{1, 0, 1}
\definecolor{colorloc}{rgb}{0.4, 0.4, 0.65}
\definecolor{colorparam}{rgb}{1, 0.6, 0.0}
\definecolor{colorvar}{rgb}{0.6, 0.7, 1}
\definecolor{colorlvar}{rgb}{0.4, 0.4, .5}
\definecolor{colordparam}{rgb}{.9, 0.8, 0.0}
\fi

\usepackage{wasysym}

\newcommand{\init}{_0}

\newcommand{\Dist}{\mathscr{D}}

\newcommand{\partfun}{\nrightarrow} 

\newcommand{\powerset}[1]{2^{#1}}

\newcommand{\KTrue}{\ensuremath{\top}}
\newcommand{\KFalse}{\ensuremath{\bot}}

\newcommand{\action}{\ensuremath{a}}
\newcommand{\actionIn}{\ensuremath{a}}
\newcommand{\actionOut}{\ensuremath{b}}
\newcommand{\actionPair}{\ensuremath{(\actionIn, \actionOut)}}

\newcommand{\word}{\textcolor{colorok}{w}}

\newcommand{\wordIn}[1][]{\textcolor{colorok}{w#1_{\mathrm{in}}}}

\newcommand{\A}{\ensuremath{\mathcal{M}}}
\newcommand{\Actions}{\Sigma}
\newcommand{\ActionsIn}{\Actions_{\mathrm{in}}}
\newcommand{\ActionsOut}{\Actions_{\mathrm{out}}}
\newcommand{\ActionsPair}{\ensuremath{\ActionsIn \times \ActionsOut}}
\newcommand{\Lg}{\mathcal{L}}
\newcommand{\loc}{s} 
\newcommand{\locinit}{\loc\init}
\newcommand{\Loc}{S} 
\newcommand{\LocFinal}{F}

\newcommand{\Edges}{E}
\newcommand{\Label}{\Lambda}
\newcommand{\run}{\rho}
\newcommand{\runs}{R}
\newcommand{\Av}{\mathrm{Av}}
\newcommand{\ActionsCtrl}{\ActionsIn^{1}}
\newcommand{\ActionsEnv}{\ActionsIn^{2}}
\newcommand{\actionCtrl}{\actionIn^{1}}
\newcommand{\actionEnv}{\actionIn^{2}}
\newcommand{\Cont}{\textsf{Cont}}
\newcommand{\Env}{\textsf{Env}}

\newcommand{\inflow}{\mathrm{inflow}}
\newcommand{\outflow}{\mathrm{outflow}}

\newcommand{\mdp}{\A}

\newcommand{\mdpEdges}{\Delta}
\newcommand{\mdpLabel}{\Lambda}
\newcommand{\mdploc}{s} 
\newcommand{\mdplocinit}{\mdploc\init}
\newcommand{\mdpLoc}{S} 
\newcommand{\mdpLocinit}{\mdpLoc\init}
\newcommand{\mdpdistinit}{\mu\init}
\newcommand{\playerStrategy}{\sigma}
\newcommand{\envStrategy}{\tau}

\newcommand{\SA}{\mathcal{A}}

\newcommand{\SAEdges}{N}

\newcommand{\SAloc}{q} 
\newcommand{\SAlocinit}{\SAloc\init}
\newcommand{\SALoc}{Q} 

\newcommand{\SATuple}{(\SALoc, \SAlocinit, \LocFinal, \Actions, \SAEdges)}

\newcommand{\Shield}{\mathcal{S}} 
\newcommand{\ShieldLoc}{\Loc_{\Shield}}
\newcommand{\shieldLocInit}{\loc_{0,\Shield}}
\newcommand{\shieldLoc}{\loc_{\Shield}}
\newcommand{\ShieldEdges}{\Delta_{\Shield}}
\newcommand{\ShieldLabel}{\Lambda_{\Shield}}

\newcommand{\G}{\ensuremath{\mathcal{G}}}
\newcommand{\GameLoc}{\ensuremath{G}}
\newcommand{\Gameloc}{\ensuremath{g}}
\newcommand{\Gamelocinit}{\Gameloc\init}
\newcommand{\Gamewin}{\ensuremath{F^{G}}}
\newcommand{\GameEdges}{\ensuremath{\Edges^{G}}}
\newcommand{\sinkLoc}{\loc_{\bot}}

\newcommand{\trainingData}{D}
\newcommand{\blue}{\mathit{blue}}
\newcommand{\red}{\mathit{red}}
\newcommand{\blueLoc}{\loc_{b}}
\newcommand{\redLoc}{\loc_{r}}
\newcommand{\mindepth}{\textsc{MinDepth}}

\usepackage{pifont}
\newcommand{\GridWorld}{\stylebenchmark{GridWorld}}
\newcommand{\WaterTank}{\stylebenchmark{WaterTank}}
\newcommand{\CliffWalking}{\stylebenchmark{CliffWalk}}
\newcommand{\CarRacing}{\stylebenchmark{CarRacing}}
\newcommand{\Taxi}{\stylebenchmark{Taxi}}
\newcommand{\SelfDrivingCar}{\stylebenchmark{SelfDrivingCar}}
\newcommand{\SideWalk}{\stylebenchmark{SideWalk}}

\newcommand{\DynamicShielding}{\stylebenchmark{Shielding}}
\newcommand{\NoShield}{\stylebenchmark{Plain}}

\newcommand{\SafePadding}{\stylebenchmark{SafePadding}}

\newcommand{\tbcolor}{\cellcolor{green!25}} 

\newcommand{\lst}{\mathrm{last}}
\newcommand{\cupdot}{\mathbin{\mathaccent\cdot\cup}}

\definecolor{vertfonce}{rgb}{0.0, 0.5, 0.0}
\definecolor{rougefonce}{rgb}{1, 0.0, 0.0}

\definecolor{cellcolor}{rgb}{.8, .8, 1}

\usepackage[colorinlistoftodos,textsize=footnotesize]{todonotes}
\ifdefined\VersionWithComments%
\newcommand{\gennote}[3]{\todo[linecolor=#2,backgroundcolor=#2!25,bordercolor=#2,author=#3]{#1}}
\else
\newcommand{\gennote}[3]{\todo[linecolor=#2,backgroundcolor=#2!25,bordercolor=#2,author=#3,disable]{#1}}
\fi

\newcommand{\tT}[1]{\gennote{#1}{blue}{TT}}
\newcommand{\sk}[1]{\gennote{#1}{green}{SK}}
\newcommand{\sP}[1]{\gennote{#1}{red}{SP}}
\newcommand{\ec}[1]{\gennote{#1}{magenta}{EC}}
\newcommand{\ih}[1]{{\gennote{#1}{purple}{IH}}}
\newcommand{\mw}[1]{\gennote{#1}{orange}{MW}}
\newcommand{\instructions}[1]{{\gennote{\bfseries #1}{red}{Instructions}}}

\newcommand{\rqanswer}[2]{\vspace{0.25em}\todo[inline,backgroundcolor=gray!30,author=\textbf{Answer to {#1}}]{#2}}

\ifdefined\VersionFinal%
\else
    \usepackage[pagewise]{lineno} 
    \linenumbers%
    
\fi

\ifdefined\VersionWithComments{}
\definecolor{colorok}{RGB}{80,80,150}
\else
\definecolor{colorok}{RGB}{0,0,0}
\fi

\newcommand{\eg}{\textcolor{colorok}{e.\,g.,}\xspace}
\newcommand{\ie}{\textcolor{colorok}{i.\,e.,}\xspace}
\newcommand{\st}{\textcolor{colorok}{s.t.}\xspace}
\newcommand{\resp}{\textcolor{colorok}{resp.}\xspace}

\usepackage{graphicx}	

\usepackage[acronyms,shortcuts,section=section,style=super,nopostdot,nogroupskip,nomain,nonumberlist]{glossaries}
\usepackage[acronym]{glossaries-extra}
\setabbreviationstyle[acronym]{long-short}
\glssetcategoryattribute{acronym}{glossdescfont}{emph}

\glsdisablehyper{}

\newacronym{edsm}    {EDSM}   {evidence-driven state merging}
\newacronym[longplural={discrete finite-state automata}]
                    {dfa}    {DFA}   {discrete finite-state automaton}
\newacronym{fsrs}    {FSRS}   {finite-state reactive system}
\newacronym{ltl}    {LTL}   {linear temporal logic}
\newacronym{mc}    {MC}   {Markov chain}
\newacronym[longplural={Markov decision processes}]
                    {mdp}    {MDP}   {Markov decision process}
\newacronym{pomdp}    {POMDP}   {partially observable Markov decision process}
\newacronym{ptmm}    {PTMM}   {prefix tree Mealy machine}
\newacronym{rpni}    {RPNI}   {regular positive and negative inference}
\newacronym{ppo}    {PPO}   {proximal policy optimization algorithm}
\newacronym{dqn}    {DQN}   {deep Q learning}

\usepackage{versions}
\ifdefined\VersionLong%
        \includeversion{LongVersionBlock}
        \excludeversion{ShortVersionBlock}
\else
        \excludeversion{LongVersionBlock}
        \includeversion{ShortVersionBlock}
\fi

\makeatletter
\def\orcidID#1{\smash{\href{https://orcid.org/#1}{\protect\raisebox{-1.25pt}{\protect\includegraphics{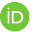}}}}}
\makeatother

\newcommand{\repeatthanks}{\textsuperscript{\thefootnote}}

\title{Dynamic Shielding for Reinforcement Learning in Black-Box Environments\thanks{\LongVersion{%
	This is the author (and extended) version of the manuscript of the same name published in the proceedings of the 20th International Symposium on Automated Technology for Verification and Analysis (ATVA 2022).
	The final version is available at \url{www.springer.com}.
	This version contains the detail of our methods, benchmarks, and experiment results.
	}}}
\author{
      Masaki Waga\orcidID{0000-0001-9360-7490}\inst{1} \and 
      Ezequiel Castellano\orcidID{0000-0002-9604-9997}\inst{2} \and 
      Sasinee Pruekprasert\orcidID{0000-0002-5929-9014}\inst{3}\thanks{The work was done during the employment of S.P. and T.T. at NII, Tokyo.}\and 
      Stefan Klikovits\orcidID{0000-0003-4212-7029}\inst{2} \and 
      Toru Takisaka\orcidID{0000-0002-5046-7480}\inst{4}\repeatthanks{}\and 
      Ichiro Hasuo\orcidID{0000-0002-8300-4650}\inst{2,5} 
}
\institute{%
    Graduate School of Informatics, Kyoto University, Kyoto, Japan
    \and
    National Institute of Informatics, Tokyo, Japan
    \and
    National Institute of Advanced Industrial Science and Technology, Tokyo, Japan
    \and
    University of Electronic Science and Technology of China, Chengdu, China
    \and
    The Graduate University for Advanced Studies, Tokyo, Japan
}

\begin{document}

\pagestyle{plain}

\maketitle

\thispagestyle{plain}

\ifdefined\VersionWithComments%
	\textcolor{red}{\textbf{This is the version with comments. To disable comments, comment out line~3 in the \LaTeX{} source.}}
\fi

\begin{abstract}
It is challenging to use reinforcement learning (RL) in cyber-physical systems due to the lack of safety guarantees during learning.
Although there have been various proposals to reduce undesired behaviors during learning, most of these techniques require prior system knowledge, and their applicability is limited.
This paper aims to reduce undesired behaviors during learning without requiring \emph{any} prior system knowledge.
%
We propose \emph{dynamic shielding}: an extension of a model-based safe RL technique called \emph{shielding} using\LongVersion{ data-driven} \emph{automata learning}.
The dynamic shielding technique constructs an approximate system model in parallel with RL using a variant of the RPNI algorithm and suppresses undesired explorations due to the shield constructed from the learned model.
Through this combination, potentially unsafe actions can be foreseen before the agent experiences them.
Experiments show that 
 our dynamic shield significantly decreases the number of undesired events during training.
\end{abstract}

\tT{hello}
\sk{hello}
\sP{hello}
\ec{hello}
\ih{hello}
\mw{hello}

\instructions{16 pages, including references}

\section{Introduction}\label{section:introduction}


\emph{Reinforcement learning} (RL)~\cite{SuttonB98} is a powerful tool for learning optimal (or near-optimal) controllers,
where the performance of controllers is measured by their long-term cumulative rewards. 
An agent in RL explores the environment by taking actions at each visited state, each of which yields a corresponding reward: 
RL aims for an efficient exploration by prioritizing actions that maximize the subsequent cumulative reward.
RL is particularly advantageous when the system model is unavailable~\cite{Mnih+15} or too large for an exhaustive search~\cite{Silver+16}.

Since an RL agent learns a controller through trial and error,
the exploration can lead to undesired behavior\LongVersion{ for the system}.
For instance, when the learning is conducted on a cyber-physical system,
such undesired behaviors can be harmful because they can damage the hardware, \eg{} by crashing into a wall.
%
%
\emph{Shielding}~\cite{AlshiekhBEKNT18}\footnote{The shield we use in this paper is the variant called \emph{preemptive shield} in~\cite{AlshiekhBEKNT18}. It is straightforward to apply our framework to the classic shield called \emph{post-posed shield}.\LongVersion{ See \cref{section:postposed_shielding} for the detail.}} is actively studied to address this problem.
A 
  shield is an external component that suggests a set of safe actions to an RL agent
 so that the agent can explore the environment with fewer encounters with undesired behaviors (\cref{figure:static_shield}).

\begin{figure*}[t]
	\centering
	\begin{subfigure}{.46\textwidth}
	\begin{tikzpicture}[font=\scriptsize\sffamily,scale=1,every node/.style={transform shape}]
		\tikzset{
			roundtangle/.style={rectangle,align=center,draw=black,fill=gray!25,rounded corners=.25em,minimum width=5em,minimum height=1.5em}
	    };

		\node[roundtangle] (env) at (0,0) {Environment};
		\node[roundtangle,right=1.2 of env] (agent) {Learning Agent};

    	\draw[-latex] ($(env.east)+(0,0.5em)$) -- node[above,pos=.3] {observation} ($(agent.west)+(0,0.5em)$);
	   	\draw[-latex] ($(env.east)-(0,0.5em)$) -- node[above,pos=.4] {reward} ($(agent.west)-(0,0.5em)$);

     	\draw[-latex] (agent.south) |- node[below, pos=.75] {\color{blue}{safe} \color{black}{action} according to $\A$} ($(env.south)+(0,-1em)$) -- (env.south);

    	\node[roundtangle,draw=blue,fill=blue!25,above=.55 of agent] (shield) {Shield $\Shield$};
    	

     	\draw[-latex,thick,blue] ($(agent.west)+(-.6em,0.5em)$) |- (shield.west);
     	\draw[-latex,thick,blue] (shield) -- node[right,align=center] {safe actions\\ according to $\A$} (agent) ;
    	    	
		\node[roundtangle,draw=blue,fill=blue!25,above right=1.0 and -1.2 of env] (system) {System Model $\A$};
		\draw[-latex,thick,blue,dashed] (system.east) -| (shield.north) node[pos=.65,above] {\emph{
 a priori construction}};
    	    	
	\end{tikzpicture} %
         \caption{Conventional shielding~\cite{AlshiekhBEKNT18} based on a system model $\A$ given by a user. The shield is constructed before starting RL.}\label{figure:static_shield}
	\end{subfigure}%
	\hfill
	\begin{subfigure}{.51\textwidth}
	\begin{tikzpicture}[font=\scriptsize\sffamily,scale=1,every node/.style={transform shape}]
		\tikzset{
			roundtangle/.style={rectangle,align=center,draw=black,fill=gray!25,rounded corners=.25em,minimum width=5em,minimum height=1.5em}
		};
		
		\node[roundtangle] (env) at (0,0) {Environment};
		\node[roundtangle,right=1.6 of env] (agent) {Learning Agent};
		
		\draw[-latex] ($(env.east)+(0,0.5em)$) -- node[above] {observation} ($(agent.west)+(0,0.5em)$);
		\draw[-latex] ($(env.east)-(0,0.5em)$) -- node[above] {reward} ($(agent.west)-(0,0.5em)$);
		\draw[-latex] (agent.south) |- node[below, pos=.75] {{\color{blue}{safe}} action according to $\A$} ($(env.south)+(0,-1em)$) -- (env.south);
		
		\node[roundtangle,draw=blue,fill=blue!25,above=.55 of agent] (shield) {Shield $\Shield$};
		\draw[-latex,thick,blue]($(agent.west)+(-.5em,0.5em)$) |- (shield.west);
		\draw[-latex,thick,blue](shield) -- node[right,align=center] {safe actions\\ according to $\A$} (agent) ;

		\node[roundtangle,draw=ForestGreen,fill=ForestGreen!25,left=1.9 of shield] (learner) {Automata Learner};
		\draw[-latex,ForestGreen,thick] ($(env.south)+(0,-1em)$) -| ($(learner.west)-(1em,0)$) -- (learner.west);
		
		\draw[-latex,ForestGreen,thick] ($(env.east)+(.3em,0.5em)$) |- ($(learner.east)-(0,0)$);

		
		\node[roundtangle,draw=ForestGreen,fill=ForestGreen!25] (system) at (1.2,1.7) {Approx. System Model $\A$};
		\draw[-latex,thick,ForestGreen,dashed] (system.east) -| (shield.north) node[pos=.75,above right = -.2 and -0.0,align=center] {\emph{regular}\\\emph{updates}};

		\draw[-latex,thick,ForestGreen,dashed] ($(learner.north)-(1em,0)$) |- (system.west) node[pos=.25,above left=-0.2 and 0.0,align=center] {\emph{regular}\\\emph{updates}};
		
	\end{tikzpicture}
         \caption{Our \emph{dynamic} shielding. An approximate model $\A$ is learned and updated in parallel with RL, and the shield is regularly updated.}\label{figure:dynamic_shield}
	\end{subfigure}
	\caption{Comparison of the conventional shielding and our dynamic shielding.
 }
	\label{figure:shielding_scheme}
\end{figure*}
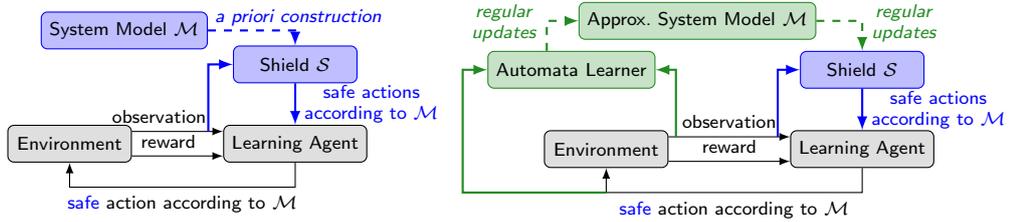

\begin{example}[\WaterTank{}]%
 \label{example:water_tank:summary}
 Consider a 100 liter water tank with valves to control the water inflow.
 The \emph{controller} opens and closes the valve, and tries to prevent the water tank from becoming empty or full.
 The exact inflow cannot be controlled but can be observed, \st{} $\inflow \in \{0,1,2\}$.
 The tank also has a random outflow, \st{} $\outflow \in \{0,1\}$.
 A good shield prevents opening (\resp{} closing) the valve when the water tank is almost full (\resp{} empty).
Moreover, there must be at least three time steps between two consecutive valve position changes to prevent hardware failure.
Hence, the shield should also prevent changing the valve position when the last change was too recent.
\end{example}
Most of the existing shielding techniques~\cite{AlshiekhBEKNT18,AvniBCHKP19,JansenKJSB20,ChengOMB19,HuntFMHDS21} for RL assume that the system model is at least partially available, and its formal analysis is feasible.
Thus, with few exceptions (\eg{}~\cite{HasanbeigAK20}), black-box systems have been beyond existing techniques.
However, this assumption of conventional shielding techniques limits the high applicability of RL, which is one of the major strengths of RL.\@

\paragraph{Dynamic shielding}

To improve the applicability of shielding for RL, we propose the dynamic shielding scheme (\cref{figure:dynamic_shield}).
Our goal is to \emph{prevent actions similar to those ones that led to undesired behavior in previous explorations}. 
In our dynamic shielding scheme, the shield is constructed and regularly updated using an approximate system model learned by a variant of \emph{the RPNI algorithm}~\cite{Oncina_1993} for passive automata learning~\cite{Lopez2016}.
Since the RPNI algorithm generates a system model consistent with the agent's experience,
a dynamic shield can prevent previously experienced undesired actions.
Moreover, since the RPNI algorithm can deem some of the actions similar, a dynamic shield can prevent undesired actions even without experiencing if the action is deemed unsafe.

It is, however, not straightforward to use a system model constructed by the original RPNI algorithm for shielding.
At the beginning of the learning, our knowledge of the system is limited, and the RPNI algorithm often deems a safe action as unsafe, which prevents necessary exploration for RL.\@
To infer the (un)safety of unexplored actions with higher accuracy, we introduce a novel variant of the RPNI algorithm tailored for our purpose. 
Intuitively, our algorithm deems two actions in the training data similar 
only if there is a long example supporting it, while the original RPNI algorithm deems two actions similar unless there is an explicit counter example.
We also modified the shield construction to 
optimistically enable not previously seen actions,
as otherwise, necessary explorations are also prevented.

We implemented our dynamic shielding scheme in Python and conducted experiments to evaluate its performance compared to two baselines: the plain RL without shielding and safe padding~\cite{HasanbeigAK20}, one of the shielding techniques applicable to black-box systems. 
Our experiments suggest that
 dynamic shielding prevents undesired exploration during training and
 often improves the quality of the resulting controller.
Although the construction and the use of dynamic shielding require some extra time, it is not prohibitive.

\paragraph{Contributions}
The following list summarizes our contributions.
\begin{itemize}
 \item We introduce the dynamic shielding scheme (\cref{figure:dynamic_shield}) using a variant of the RPNI algorithm for passive automata learning.
 \item We modify the RPNI algorithm and the shield construction so that the shield does not prevent necessary exploration, even if our prior system knowledge is limited.
 \item We experimentally show that our dynamic shielding scheme significantly reduces the number of undesired explorations during training.
\end{itemize}




\subsection{Related works}\label{section:related_work}

The notion of shield is originally proposed in~\cite{BloemKKW15} as an approach for \emph{runtime enforcement}. In this line of research \cite{BharadwajBDKT19,BloemJKLLPpreprint,WWDW19}, a shield takes the role of an \emph{enforcer} that overwrites the output of the system when the specification is violated at runtime. Shielding in this context is fundamentally different from ours, where shields are used to block system inputs that incur unsafe outputs of the system.

Shielding for RL (or simply shielding) is categorized as a technique of \emph{safe RL}.
Using the taxonomy of~\cite{GarciaF15}, shielding is an instance of “teacher provides advice”; i.e., a shield as a teacher giving additional information to the learning agent to prevent unsafe exploration.
Such a use of a shield is first proposed in~\cite{AlshiekhBEKNT18}, and several probabilistic variants\LongVersion{ of shielding} are also proposed~\cite{JansenKJSB20,AvniBCHKP19,BoutonKNFKTpreprint}; they assure that the learning is safely done with high probability (but not necessarily with full certainty). 
In these works, a system model is necessary to construct a shield. 
Some works propose shielding for inaccurate models~\cite{PrangerKTD0B21,HuntFMHDS21,ChengOMB19}, but they still require some prior knowledge of the system (\eg{} the nominal dynamics of the system).


To the best of our knowledge, the existing work closest to ours is \emph{cautious RL}~\cite{HasanbeigAK20}:
it is also a shielding-based safe RL that does not require a system model (but can perform better with a system model).
To avoid unsafe events under uncertainty,
cautious RL learns an MDP in parallel with the RL process, and
its \emph{safe padding} blocks actions that let the agent come too close to an area where
an action may lead to undesired behavior according to the learned MDP.\@
That strong blocking policy works as a safety buffer against unexpected transitions. 

A major difference between our technique and safe padding is in the approximate model learning:
in safe padding, the observation space is directly used as the state space of the learned MDP, while we merge some of them based on similarity of the suffixes to generalize observations.
\LongVersion{It is demonstrated in~\cite{HasanbeigAK20} that safe padding offers a promising safety assurance when prior knowledge of the system (\eg{} the system's dynamics without disturbance) is available.
However, it is unclear if, without generalization, it also reduces safety violations during learning 
 when such knowledge is unavailable in the beginning.}
In experiments,  
we demonstrate that generalization by automata learning effectively reduces safety violations.

\begin{figure}[tbp]
    \centering
\begin{tikzpicture}
\tikzset{
    rect/.append style={draw,fill=white,thick,font=\scriptsize,rectangle,inner sep=3pt,minimum height=.5cm, minimum width=3cm,align=center}
}

\node[rectangle,dotted,fill=gray!75,opacity=.25,minimum height=3.5cm,minimum width=11.5cm] at (2.5,-.85) {};

\node[rect] (fsrs) {System Abstraction $\A$\\(as \emph{\acs{fsrs}} in \cref{def:FSRS})};
\node[rect,right=of fsrs] (sa) {Specification $\SA^{\varphi}$\\(as \emph{Safety Automaton} in \cref{def:SA})};
\node[rect,below right=.5 and 0.5 of fsrs.south] (sg) {\emph{Safety Game} $\G$ in \cref{def:safety_game}};
\node[rect,below=.5 of sg] (shield) {Shield $\Shield$ (as \emph{Mealy machine} in \cref{def:mealy_machine})};

\node[rect,draw=gray!50,above=0.7 of fsrs] (sys) {System $\A^{\envStrategy}$\\(representable as an \acs{mdp} in \cref{def:MDP})};
\node[rect,draw=gray!50,above=0.7 of sa] (spec) {Specification $\varphi$\\(\eg in LTL)};

\draw[->] (fsrs) -- (sg);
\draw[->] (sa) -- (sg);
\draw[->] (sg) -- (shield) node[pos=.5,right] {\footnotesize \emph{game solving}};

\draw[<->,dashed] (sys) -- (fsrs);
\draw[->,dashed] (spec) -- (sa);


\node[color=black!50,above=1.25 of sg] {\textbf{Shield Construction}};

\end{tikzpicture}    
\caption{Shield construction schema. Our main contribution is the system abstraction's inference through automata learning techniques.}%
\label{fig:flowschema}
\end{figure}
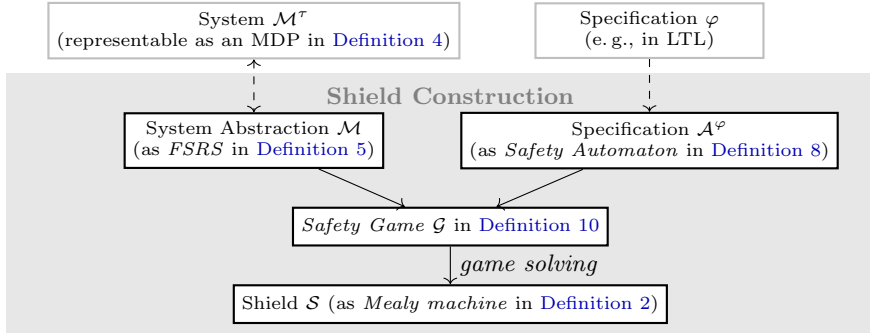

\section{Preliminaries}\label{section:preliminaries}
For a set $X$, we denote the set of probability functions over $X$ by $\Dist X$.
For a set $\Actions$, let $\Actions^*$ be the set of finite sequences over $\Actions$.
We denote the empty sequence by $\varepsilon$.
For any $\word \in \Actions^*$, the length is denoted by $|\word|$.

We use \emph{Mealy machines} to formalize an abstract system model and a shield.

\begin{definition}
 [Mealy machine]\label{def:mealy_machine}
    A \emph{Mealy machine} is a 6-tuple
    $\A = (\Loc, \locinit, \ActionsIn, \ActionsOut, \Edges, \Label)$, where:
    \begin{myitemize}
        \item $\Loc$ is a finite set of states,
        \item $\locinit \in \Loc$ is the initial state,
        \item $\ActionsIn$ and $\ActionsOut$ are finite alphabets for inputs and outputs\LongVersion{, respectively},
        \item $\Edges \colon \Loc \times \ActionsIn \partfun \Loc$ is a partial \emph{transition function}, and
        \item $\Label \colon \Loc \times \ActionsIn \partfun \ActionsOut$ is a partial \emph{output function} such that $\Label(\loc, \action)$ is defined if and only if $\Edges(\loc, \action)$ is defined.
    \end{myitemize}
\end{definition}

\begin{definition}
 [path, run, output]\label{def:run}
   Let $\A = (\Loc, \locinit, \ActionsIn, \ActionsOut, \Edges, \Label)$ be a Mealy machine and
   let $\wordIn = \actionIn_1,\actionIn_2, \dots \actionIn_n \in \ActionsIn^*$ be an input word.
    For a state $\loc \in \Loc$ of $\A$,
    the \emph{path} $\run$ of $\A$ from $\loc$ over $\wordIn$
    is the alternating sequence $\run = \loc, \action_1, \loc_1, \action_2, \dots, \action_n, \loc_n$ 
    of states $\loc_i \in \Loc$ and input actions $\actionIn_i \in \ActionsIn$ satisfying $\Edges(\loc, \actionIn_1) = \loc_1$ and $\Edges(\loc_{i-1}, \actionIn_i) = \loc_i$ for each $i \in \{2,\dots,n\}$.
    A \emph{run} of a Mealy machine $\A$ is a path of $\A$ from the initial state $\locinit$.
    For a state $\loc \in \Loc$ of $\A$, 
    we write $\Edges(\wordIn, \loc)$ to denote the last state $\loc_n$ of the path $\run = \loc, \action_1, \loc_1, \action_2, \dots, \action_n, \loc_n$ of $\A$ from $\loc$ over $\wordIn$.
    The \emph{output} $\A(\wordIn,\loc) \in \ActionsOut$ of a Mealy machine $\A$ is defined by $\A(\wordIn,\loc) = \Label(\loc_{n-1}, \actionIn_n)$, where
    $\loc_{n-1} = \Edges(\wordIn', \loc)$ and 
    $\wordIn' = \actionIn_1, \actionIn_2, \dots, \actionIn_{n-1}$.
    We write $\Edges(\wordIn) = \Edges(\wordIn, \locinit)$ and $\A(\wordIn) = \A(\wordIn, \locinit)$.
\end{definition}

\subsection{Automata and games for system modeling}\label{subsection:games}


As shown in \cref{fig:flowschema}, we assume that the system is representable as a \ac{mdp}, and
we use \acp{fsrs}~\cite{AlshiekhBEKNT18} to abstract the \ac{mdp}.
More precisely, an \ac{fsrs} is a two-player deterministic game of the controller (\Cont{}) and the environment (\Env{}), where \acp{mdp}' probabilistic transitions are represented by \Env{} transitions.
Note that \acp{mdp} are only used for the theoretical discussion in this paper.

\begin{definition}
    [\Acf{mdp}]\label{def:MDP}
    An \ac{mdp} is a 5-tuple
    $(\Loc, \locinit, \ActionsIn, \ActionsOut, \mdpEdges)$ such that:
    \begin{myitemize}
        \item $\Loc$ is a finite set of states,
        \item $\locinit \in \Loc$ is the initial state,
        \item $\ActionsIn$ and $\ActionsOut$ are the finite set of input and output alphabets, respectively, and
        \item $\mdpEdges \colon \Loc \times \ActionsIn \partfun \Dist(\Loc \times \ActionsOut)$ is the probabilistic  \emph{transition function}.
    \end{myitemize}
\end{definition}

\begin{definition}
    [\acf{fsrs}]\label{def:FSRS}
    An \ac{fsrs} is a Mealy machine
    $\A = (\Loc, \locinit, \ActionsIn, \ActionsOut, \Edges, \Label)$ that satisfies the following:
    \begin{myitemize}
        \item $\ActionsIn = \ActionsCtrl \times \ActionsEnv$, where
        $\ActionsCtrl$ (\resp{} $\ActionsEnv$) is the set of actions of \Cont{} (\resp{} \Env{});
        \item for each $\loc \in \Loc$ and $\actionCtrl \in \ActionsCtrl$, there is $\actionEnv \in \ActionsEnv$ for which
        $\Edges(\loc, (\actionCtrl, \actionEnv))$ is defined.
    \end{myitemize}
\end{definition}

\begin{definition}
 [strategy]\label{def:strategy}
 For an \ac{fsrs}
 $\A = (\Loc, \locinit, \ActionsCtrl \times \ActionsEnv, \ActionsOut, \Edges, \Label)$,
 \emph{strategies} of \Cont{} and \Env{} are functions
 $\playerStrategy \colon \Pi_\A \to \Dist \ActionsCtrl$ and
 $\envStrategy \colon \Pi_\A \times \ActionsCtrl \to \Dist \ActionsEnv$, respectively, where $\Pi_\A$ is the set of runs of $\A$\LongVersion{\footnote%
 {
    This definition implies that we understand \acp{fsrs}  
    as turn-based games rather than concurrent ones.  
 }}.
 For strategies $\envStrategy$ of \Env{}, we also require that $\tau(\run,\actionCtrl) (\{\actionEnv\}) > 0$ holds only if $\Edges(\run_\lst, (\actionCtrl,\actionEnv))$ is defined, where $\run_\lst$ is the last state of $\run$.
 A strategy is \emph{memoryless} if it is independent of the run except for the last state.
\end{definition}

We use an \ac{fsrs} as an abstraction of an \ac{mdp} because for an \ac{fsrs} $\A$ and a memoryless strategy $\envStrategy$ of \Env{}, there is a canonical \ac{mdp} $\A^{\envStrategy}$, where the actions of \Env{} are chosen by $\envStrategy$.
Formally, for $\A = (\Loc, \locinit, \ActionsCtrl \times \ActionsEnv, \ActionsOut, \Edges, \Label)$ and $\envStrategy$,
$\A^{\envStrategy}$ is $\A^\envStrategy = (\mdpLoc, \mdplocinit, \ActionsCtrl, \ActionsOut, \mdpEdges)$, where
for each $\loc, \loc' \in \Loc$, $\actionCtrl \in \ActionsCtrl$, and $\actionOut \in \ActionsOut$, $\mdpEdges(\mdploc, \actionCtrl)$ is such that
\[
  (\mdpEdges(\mdploc, \actionCtrl))(\mdploc', \actionOut) = (\envStrategy(\mdploc, \actionCtrl))(\{\actionEnv \in \ActionsEnv \mid \Edges(\mdploc,(\actionCtrl, \actionEnv))= \mdploc' \land \Label(\mdploc,(\actionCtrl, \actionEnv))= \actionOut \}).
\]

By fixing both \Cont{} and \Env{} strategies $\playerStrategy$ and $\envStrategy$ of an \ac{fsrs} $\A$, we obtain a\LongVersion{ (purely)} stochastic structure $\A^{\playerStrategy, \envStrategy}$.
We define the language $\Lg(\A^{\playerStrategy, \envStrategy}) \subseteq {((\ActionsCtrl \times \ActionsEnv) \times \ActionsOut)}^{*}$ of $\A^{\playerStrategy, \envStrategy}$ as the set of sequences of input/output actions $((\actionIn^1_i, \actionIn^2_i), \actionOut_i) \in (\ActionsCtrl \times \ActionsEnv) \times \ActionsOut$ in the runs of $\A^{\playerStrategy, \envStrategy}$.

\begin{example}[\WaterTank{} \ac{fsrs}]%
 \label{example:water_tank:fsrs}
 The Water Tank in \cref{example:water_tank:summary} is formalized as an \ac{fsrs} $\A = (\Loc, \locinit, \ActionsCtrl \times \ActionsEnv, \ActionsOut, \Edges, \Label)$, where:
 \begin{myitemize}
    \item $\mdpLoc = \{0, 1, \dots, 100 \}$;
    \item $\ActionsCtrl = \{\mathrm{open}, \mathrm{close}\}$;
    \item $\ActionsEnv = \inflow \times \outflow$, where $\inflow = \{0,1,2\}$ and $\outflow = \{0,1\}$;
    \item $\ActionsOut = \{\mathrm{low}, \mathrm{safe}, \mathrm{high}\}$;
    \item $\mdpEdges(\mdploc, (\actionCtrl, (n,m)))$ is defined if either $\actionCtrl = \mathrm{open}$ and $n \in \{1,2\}$ or $\actionCtrl = \mathrm{close}$ and $n = 0$;
    \item $\mdpEdges(\mdploc, (\actionCtrl, (n,m))) = \max\{0, \min\{\mdploc + n - m, 100\}\}$;
    \item For any $(\actionCtrl,(n,m)) \in \ActionsCtrl \times \ActionsEnv$, we have $\mdpLabel(0, (\actionCtrl, (n,m))) = \mathrm{low}, \mdpLabel(100, (\actionCtrl, (n,m))) = \mathrm{high}$, and $\mdpLabel(\mdploc, (\actionCtrl, (n,m))) = \mathrm{safe}$ otherwise.
 \end{myitemize}

The probabilistic behavior of the \WaterTank{} environment (\ie{} \Env{}) is such that
 \begin{ienumeration}
  \item the inflow of water is randomly chosen from $\{1,2\}$ when $\actionCtrl = \mathrm{open}$, and
  \item the outflow of the water is randomly chosen from $\{0,1\}$. 
 \end{ienumeration}
Such a behavior is formalized by a \Env{} strategy $\envStrategy$ such that:
\begin{myitemize}
    \item $\envStrategy(\mdploc,\mathrm{open})(\{(n,m)\}) = 0.25$ for each $(n,m) \in \{1,2\} \times \{0,1\}$, and
    \item $\envStrategy(\mdploc,\mathrm{close})(\{(0,m)\}) = 0.5$ for each $m \in \{0,1\}$.
\end{myitemize}
\end{example}

\subsection{Safety automata for specifications}\label{section:specifications}
In the shielding methodology, the shield's specification is given as a \emph{safety automaton}.
As shown in \cref{fig:flowschema}, typically, this automaton is automatically generated from a temporal logic formula, \eg{} \ac{ltl}.
See, \eg{}~\cite{KL06} for the construction of an automaton from an \ac{ltl} formula.


\begin{definition}[Safety Automata]\label{def:SA}
    A \emph{safety automaton} is a 5-tuple
    $\SA = \SATuple$, where:
    \begin{myitemize}
        \item $\SALoc$ is a finite set of states,
        \item $\SAlocinit \in \SALoc$ is the initial state,
        \item $\LocFinal \subseteq \SALoc$ is the set of safe states,
        \item $\Actions$ is the finite set of alphabet, and
        \item $\SAEdges \colon \SALoc \times \Actions \to \SALoc$ is the \emph{transition function} satisfying $\SAEdges(\SAloc, \action) \in \LocFinal$ only if $\SAloc \in \LocFinal$.
    \end{myitemize}
\end{definition}

A run of a safety automaton is defined similarly to that of a Mealy machine.
For a safety automaton $\SA$, the \emph{language} $\Lg(\SA) \subseteq \Actions^*$ of $\SA$ is the set of words $\word = \action_1, \action_2, \dots, \action_n$ such that
the\LongVersion{ unique} run $\SAloc_0, \action_1,\SAloc_1,\ldots,\action_n,\SAloc_n$ over $\word$ satisfies $\SAloc_n \in \LocFinal$.
By the definition of $\SAEdges$ and $\LocFinal$, $\Lg(\SA)$ is prefix-closed.
We call $\varphi \subseteq \Actions^*$ a \emph{satefy specification} if there is a safety automaton recognizing it. 
For an \ac{fsrs} $\A$ over $\ActionsIn$ and $\ActionsOut$, a \Cont{} strategy $\playerStrategy$ of $\A$, and a specification $\varphi \subseteq {(\ActionsPair)}^{*}$,
we say $\A$ \emph{satisfies} $\varphi$ under $\playerStrategy$ if for any \Env{} strategy $\envStrategy$, we have $\Lg(\A^{\playerStrategy, \envStrategy}) \subseteq \varphi$.


\subsection{Shielding for safe reinforcement learning}%
\label{subsection:shielding}

We use an \ac{fsrs} and a safety automaton to define a \emph{safety game} which we use to create a shield for safe RL.\@
First, we show the formal definition of shields.

\begin{definition}
    [(Preemptive) Shield~\cite{AlshiekhBEKNT18}]\label{def:preemptive_shield}
    Let  $\A$ be an \ac{fsrs} as above.\@
    A \emph{shield} for $\A$ is a Mealy machine $\Shield = (\ShieldLoc, \shieldLocInit, \ActionsCtrl \times \ActionsEnv, \powerset{\ActionsCtrl}, \ShieldEdges, \ShieldLabel)$ \st{}
    for any $\shieldLoc \in \ShieldLoc$ and input actions $\action, \action' \in \ActionsCtrl \times \ActionsEnv$,
    we have $\ShieldLabel(\shieldLoc, \action) = \ShieldLabel(\shieldLoc, \action')$.
\end{definition}

For any input word $\wordIn \in {(\ActionsCtrl \times \ActionsEnv)}^*$, the shield $\Shield$ returns $\ShieldLabel(\wordIn)$ as the set of safe actions for \Cont{} after $\wordIn$ is processed in $\A$.
A shield $\Shield$ canonically induces a strategy $\sigma$ for an \ac{mdp} $\mdp^\tau$ generated by any $\tau$, namely a strategy $\sigma$ such that 
$\sigma(\hat\loc_0, a_1, \hat\loc_1, \ldots, a_k, \hat\loc_k)$ is the discrete uniform distribution over $\Label_\mathcal{S}(\wordIn)$.

Given an \ac{fsrs} $\A$ and a specification $\varphi$ realized by a safety automaton $\SA^\varphi$ (\ie{} $\Lg(\SA^\varphi) = \varphi$),
our goal is to construct a shield $\mathcal{S}$ such that 
for any \Env{}  strategy $\envStrategy$, the \ac{mdp} $\mdp^\tau$ satisfies $\varphi$ under any \Cont{} strategy $\sigma$ compatible with $\mathcal{S}$.
To this end, we consider a \emph{safety game} constructed by $\A$ and $\SA^\varphi$. 

\begin{definition}
 [Safety Game]\label{def:safety_game}
    A \emph{2-player safety game} is an \ac{fsrs}
    $\G = (\GameLoc, \Gamelocinit, \ActionsCtrl \times \ActionsEnv, \{0,1\}, \GameEdges, \Gamewin)$.
\end{definition}

For an \ac{fsrs} $\A$ and a safety automaton $\SA^{\varphi}$,
the parallel composition $\A \mathrel{||} \SA^{\varphi}$ is the safety game
$\A \mathrel{||} \SA^{\varphi} = (\GameLoc, (\loc, \SAloc), \ActionsCtrl \times \ActionsEnv, \{0,1\}, \GameEdges, \Gamewin)$, where
$\GameLoc = \Loc \times \SALoc$,
$\GameEdges ((\loc, \SAloc), (\actionIn^1, \actionIn^2)) = (\loc', \SAloc')$, $\loc' = \Edges(\loc, (\actionIn^1, \actionIn^2))$, $\SAloc' = \SAEdges(\SAloc, (\actionIn, \Label(\loc, (\actionIn^1, \actionIn^2))))$, and
$\Gamewin$ is such that $\Gamewin((\loc,\SAloc), (\actionIn^1, \actionIn^2)) = 1$ if and only if $\SAEdges(\SAloc, ((\actionIn^1, \actionIn^2), \Label(\loc, (\actionIn^1, \actionIn^2)))) \in \LocFinal$.


We say $\Gameloc \in \GameLoc$ is a \emph{safe state} if there is a \Cont{} strategy $\playerStrategy$ such that,
for any \Env{} strategy $\envStrategy$ and $\word = ((\actionIn^1_1, \actionIn^2_1), \actionOut_1), ((\actionIn^1_2, \actionIn^2_2), \actionOut_2), \dots, ((\actionIn^1_n, \actionIn^2_n), \actionOut_n) \in \Lg(\A^{\sigma, \tau})$, the unique run $\Gameloc_0, (\actionIn^1_1, \actionIn^2_1),\Gameloc_1, \dots, (\actionIn^1_n, \actionIn^2_n), \Gameloc_n$ of $\A$ over $\word$ satisfies $\Gameloc_n \in \Gamewin$.
Intuitively, a state $\Gameloc$ is safe if there is a \Cont{} strategy $\playerStrategy$ such that the run always remains in safe states, regardless of the \Env{} actions.

We utilize the shield construction algorithm in~\cite{AlshiekhBEKNT18} to the above safety game $\G$.
Namely, the shield generated from $\G$ is an \ac{fsrs} $\Shield = (\GameLoc, \Gamelocinit, \ActionsCtrl \times \ActionsEnv, \powerset{\ActionsCtrl}, \GameEdges, \ShieldLabel)$, where
for each state $\loc \in \GameLoc$, $\ShieldLabel$ assigns the set of \Cont{} actions such that
for any \Env{} strategy, we remain in the safe states forever.
\subsection{The RPNI algorithm for passive automata learning}%
\label{section:passive_automata_learning}
We use a variant of the \LongVersion{\ac{rpni}}\ShortVersion{\acs{rpni}} algorithm~\cite{Oncina_1993} to learn an approximate system model in parallel with RL.\@  
See, \eg{}~\cite{Lopez2016} for the detail.
%
Given a finite training data $\trainingData \subseteq \ActionsIn^* \times \ActionsOut$,
the \LongVersion{\ac{rpni}}\ShortVersion{\acs{rpni}} algorithm constructs a Mealy machine $\A$ that is \emph{consistent} with the training data $\trainingData$,
\ie{} for any $(\wordIn, \actionOut) \in \trainingData$, we have $\A(\wordIn) = \actionOut$.
For simplicity, we assume that the training data is prefix-closed, \ie{}
if $\trainingData$ contains $(\wordIn, \actionOut)$,
for any prefix $\wordIn[']$ of $\wordIn$
and for some $\actionOut' \in \ActionsOut$,
$\trainingData$  contains $(\wordIn[], \actionOut')$.
%
This assumption holds in our dynamic shielding scheme and does not harm its applicability.
Since we learn a Mealy machine, we assume that the output in the training data is uniquely determined by the input word,
\ie{} for each $(\wordIn, \actionOut),(\wordIn['], \actionOut') \in \trainingData$,
$\wordIn = \wordIn[']$ implies $\actionOut = \actionOut'$.

The \LongVersion{\ac{rpni}}\ShortVersion{\acs{rpni}} algorithm creates the \ac{ptmm} $\A_{\trainingData}$ from the training data $\trainingData$ and constructs a Mealy machine $\A$ by merging the states of $\A_{\trainingData}$.
The \ac{ptmm} $\A_{\trainingData}$ is the Mealy machine such that the states are the input words in the training data, and the transition and output functions are $\Edges(\wordIn, \actionIn) = \wordIn \cdot \actionIn$ and $\Label(\wordIn, \actionIn) = \actionOut$
if $(\wordIn\cdot\actionIn, \actionOut) \in \trainingData$ for some $\actionOut \in \ActionsOut$, and otherwise, undefined.
\SetKwFunction{FConsistent}{compatible}
When merging states $\loc$ and $\loc'$ of $\A_{\trainingData}$,
we require them to be \emph{compatible}, \ie{} the merging of $\loc$ and $\loc'$ must not cause any nondeterminism to
make the learned Mealy machine the consistent with the training data.
The \LongVersion{\ac{rpni}}\ShortVersion{\acs{rpni}} algorithm greedily merges states $\loc$ and $\loc'$ of $\A_{\trainingData}$ as far as they are compatible.
%

\section{Dynamic shielding with online automata inference}\label{section:our_method}
Here, we introduce our dynamic shielding scheme in \cref{figure:dynamic_shield}, where the shield is constructed from the \ac{fsrs} inferred in parallel with the RL process.
Since our dynamic shielding scheme does not require the system model, we can apply it to black-box systems.


\subsection{Dynamic shielding scheme}

In conventional shielding (\cref{figure:static_shield}), the shield is constructed from the system model and provides safe actions to the learning agent.
See also \cref{fig:flowschema} for the shield creation schema.
By shielding, we can ensure the safety of the exploration in RL if the given system model correctly abstracts the actual system.
However, the use of a predefined system model limits the applicability of shielding.
Specifically, we cannot use the conventional shielding scheme for black-box systems.

In dynamic shielding (\cref{figure:dynamic_shield}), using the \LongVersion{\ac{rpni}}\ShortVersion{\acs{rpni}} algorithm, we learn a system model in parallel with RL, and prevent exploring undesired actions according to the learned model.
More precisely, all the inputs and the outputs in RL are given to the \LongVersion{\ac{rpni}}\ShortVersion{\acs{rpni}} algorithm, and we obtain an \ac{fsrs} $\A$ in \cref{fig:flowschema}.
From $\A$ and the given specification $\varphi$, we generate a shield $\Shield$, and prevent undesired exploration using $\Shield$.
We continuously reconstruct the shield along with the RL process.
Since we learn \ac{fsrs} from observations, the learned \ac{fsrs} may be incomplete or inconsistent with the actual system, which causes the following challenges.



\subsection{Challenge 1: Incompleteness of the learned \ac{fsrs}}\label{subsec:sink_state}

Since the \LongVersion{\ac{rpni}}\ShortVersion{\acs{rpni}} algorithm is based on merging of nodes of the prefix tree representing the training data,
 the inferred \ac{fsrs} $\A$ has partial transitions when there are unexplored actions from some of the states of $\A$.
Specifically, there may be a state $\loc \in \Loc$ of the \ac{fsrs} $\A$ and a \Cont{} action $\actionIn^1$ such that
$\Edges(\loc, (\actionIn^1, \actionIn^2))$ is undefined for any \Env{} action $\actionIn^2$.
If we construct a shield from such an \ac{fsrs} by the algorithm in~\cite{AlshiekhBEKNT18},
the shield prevents the \Cont{} action $\actionIn^1$
because there is no transition labeled with $\actionIn^1$ to stay in the safe states.
However, since we have no evidence of safety violation by the \Cont{} action $\actionIn^1$, this interference is against the ``minimum interference'' policy of shielding~\cite{AlshiekhBEKNT18}. 
Moreover, such interference can harm the performance of the synthesized controller because it limits the exploration in RL.

To minimize the interference, we modify the safety game construction so that the undefined destinations in the \ac{fsrs} are deemed safe.
More precisely, we create a fresh sink state $\sinkLoc$ that is safe in the safety game, and make it the destination of the undefined transitions in the \ac{fsrs}.
We remark that the use of such an additional sink state does not allow any \Cont{} action with evidence of safety violation because even if a \Cont{} action $\actionIn^1$ leads to a safe state for one \Env{} action $\actionIn^2$, the \Cont{} action $\actionIn^1$ is prevented if there is another \Env{} action $\tilde{\actionIn}^2$ leading to a violation of the specification $\varphi$.

The safety game construction is formalized as follows.
\begin{definition}
For an \ac{fsrs} $\A = (\Loc, \locinit, \ActionsCtrl\times\ActionsEnv, \ActionsOut, \Edges, \Label)$ and 
a safety automaton $\SA^\varphi = (\SALoc^\varphi, \SAlocinit^\varphi, \LocFinal, \ActionsOut, \SAEdges^\varphi)$,
their compositions is the safety game $\G = (\GameLoc, \Gamelocinit, \ActionsCtrl \times \ActionsEnv, \{0,1\}, \GameEdges, \Gamewin)$ such that:
\begin{myitemize}
 \item
 $\GameLoc = (\Loc \cup \{\sinkLoc\}) \times \SALoc^\varphi$;
 \item
 $\Gamelocinit = (\locinit, \SAlocinit^\varphi)$;
 \item
 $\GameEdges$ is
       $\GameEdges((\loc,\SAloc^\varphi), (\actionIn^1,\actionIn^2)) = (\Edges(\loc,(\actionIn^1,\actionIn^2)),$ $\SAEdges^\varphi(\SAloc^\varphi, \Label(\loc, (\actionIn^1,\actionIn^2))))$ if $\Edges(\loc,(\actionIn^1,\actionIn^2))$, is defined, and otherwise,
       $\GameEdges((\loc,\SAloc^\varphi), (\actionIn^1,\actionIn^2)) = (\sinkLoc, \SAEdges^\varphi(\SAloc^\varphi, \Label(\loc, (\actionIn^1,\actionIn^2))))$;
 \item
 $\Gamewin = (\Loc \times \LocFinal) \cup( \{\sinkLoc\} \times \SALoc^\varphi)$.
\end{myitemize}
\end{definition}




\subsection{Challenge 2: Precision in automata learning}\label{subsection:min_depth}

In the \LongVersion{\ac{rpni}}\ShortVersion{\acs{rpni}} algorithm, the generalization of the training data is realized by the state merging.
Such generalization allows a dynamic shield to foresee potentially unsafe actions before the learning agent experiences them.

Since the \LongVersion{\ac{rpni}}\ShortVersion{\acs{rpni}} algorithm aims to construct a minimal \ac{fsrs}, it greedily merges the states as long as there is no evidence of inconsistency.
However, when the training data is limited,
there may not be evidence of the inconsistency of states, even if they must be distinguished according to the (black-box) ground truth.
In such a case,
the greedy merging in the \LongVersion{\ac{rpni}}\ShortVersion{\acs{rpni}} algorithm may decrease the precision of the learned \ac{fsrs}.
This is especially the case at the beginning of the training because the training data $\trainingData$ is small.
%
%
Moreover, such an imprecise dynamic shield may even harm the quality of the controller synthesized by RL because the dynamic shield may prevent necessary exploration when it deems a safe action to be unsafe.






To prevent such too aggressive merging, we require additional evidence in the state merging.
Namely, we modify the \LongVersion{\ac{rpni}}\ShortVersion{\acs{rpni}} algorithm so that states $\loc$ and $\loc'$ can be merged only if
 there are paths from $\loc$ and $\loc'$ over a common word $\word \in {(\ActionsIn \times \ActionsOut)}^*$ longer than a threshold $\mindepth{}$.
 This idea is related to the \ac{edsm}~\cite{LPP98}, where similar evidence is used in prioritizing the merged states. 
In our implementation, we adaptively decide\LongVersion{ the threshold} $\mindepth$ so that $\mindepth{}$ is large when the mean length of the episodes is short.
This typically makes the merging more aggressive through the learning process.
Nevertheless, further investigation of $\mindepth$ is future work.

\subsection{Theoretical validity of our dynamic shielding}

We show that our dynamic shielding scheme assures the safety of RL when the training data is large enough to construct an \emph{abstraction} of the actual system.
%


\begin{definition}
	[abstraction]\label{def:simulation}
 For \acp{fsrs}
 $\A = (\Loc, \locinit, \ActionsCtrl \times \ActionsEnv, \ActionsOut, \Edges, \Label)$ and
 $\A' = (\Loc', \locinit', \ActionsCtrl \times \ActionsEnv, \ActionsOut, \Edges', \Label')$,
 $\A'$ \emph{abstracts} $\A$ if for each $w \in {(\ActionsCtrl \times \ActionsEnv)}^*$ if $\Edges(w)$ is defined, $\Edges'(w)$ is also defined, and $\Label(w) = \Label'(w)$ holds.
\end{definition}

\begin{theorem}[safety assurance by a dynamic shield]%
 \label{theorem:validity}
 Let $\A$ be an \ac{fsrs}, 
 let $\A'$ be its abstraction, 
 let $\varphi$ be a specification realized by $\SA^\varphi$, and 
 let $\Shield$ be a shield generated by $\A'$ and $\SA^\varphi$.
 For any strategy $\envStrategy$ of \Env{} in $\A$,
 the MDP $\mdp^\envStrategy$ satisfies $\varphi$ under any strategy $\playerStrategy$ generated by $\Shield$.
\end{theorem}

\begin{proof}[sketch]
Let $\playerStrategy$ be a \Cont{} strategy of $\A' \times \SA^\varphi$, and let $\playerStrategy'$ be a \Cont{} strategy of $\A'$ obtained by taking the obvious projection of $\sigma$.
It is well known that if a state $(\loc,\SAloc)$ of the product game is winning for \Cont{} under $\playerStrategy$, in $\A'$,
any path from $s$ satisfies $\varphi$ under $\sigma'$.
Since $\A'$ is an abstraction of $\A$, if any path from the initial state $s'_0$ of $\A'$ satisfies $\varphi$ under $\sigma'$,
any path from the initial state $s_0$ of $\A$ also satisfies $\varphi$ under $\sigma'$.
Therefore, for any \Env{} strategy $\envStrategy$, the MDP $\mdp^\envStrategy$ satisfies $\varphi$ under any strategy $\playerStrategy$ generated by $\Shield$.
 \qed{}
\end{proof}

%

Let $\A^{\envStrategy}$ be the MDP representing the actual system in \cref{fig:flowschema}, where $\A$ is an \ac{fsrs} and $\envStrategy$ is a \Env{} strategy in $\A$.
By \cref{theorem:validity}, if the learned \ac{fsrs} $\A'$ abstracts\LongVersion{ the \ac{fsrs}} $\A$, 
the dynamic shield $\Shield$ assures the safety of the exploration in $\A^{\envStrategy}$.
When the training data is large and includes a certain set of words called \emph{characteristic} set, the \LongVersion{\ac{rpni}}\ShortVersion{\acs{rpni}} algorithm is guaranteed to learn the abstraction correctly~\cite{Lopez2016}.
However, in our dynamic shielding scheme, the training data may not be a superset of the characteristic set even in the limit because 
the dynamic shield interferes with the exploration.
%
Nevertheless, 
suppose the learning algorithm eventually explores all available actions, and the maximum length of each episode in RL is long enough to cover all states of $\A$. In that case,
the training data eventually includes the characteristic set of $\A$ restricted to the safe actions according to the dynamic shield $\Shield$.
Since the dynamic shield constructed from such training data prohibits all the unsafe actions
,
our dynamic shielding assures the safety in the limit.

\section{Experimental evaluation}\label{section:experiments}

To show the applicability of dynamic shielding and the viability of our approach,
we conducted a series of experiments. The following research questions guided our experiments on our dynamic shielding scheme for RL.\@
\begin{description}
 \item[RQ1] Does dynamic shielding reduce the number of undesired behaviors?
 \item[RQ2] How does dynamic shielding affect the quality of the controller synthesized by RL?
 \item[RQ3] What is the computational overhead of dynamic shielding, and is it prohibitively large?
\end{description}


\subsection{Implementation and experiments}
We implemented our dynamic shielding scheme in Python 3 and Java using 
LearnLib~\cite{DBLP:conf/cav/IsbernerHS15} for the \LongVersion{\ac{rpni}}\ShortVersion{\acs{rpni}} algorithm\footnote{The artifact is publicly available at \url{https://doi.org/10.5281/zenodo.6906673}.}.
For deep RL, we used Stable Baselines 3~\cite{stable-baselines3} (for \ac{ppo}~\cite{DBLP:journals/corr/SchulmanWDRK17}) and Keras-RL~\cite{plappert2016kerasrl} (for \ac{dqn}\LongVersion{ with Boltzmann exploration}~\cite{DBLP:journals/corr/MnihKSGAWR13}).
We used two libraries and learning algorithms to demonstrate that\LongVersion{ our} dynamic shielding\LongVersion{ scheme} is independent of the RL algorithm.

To answer the research questions, we compared the performance of RL with dynamic shielding (denoted as \DynamicShielding{})
with the standard RL process without shielding (denoted as \NoShield{}) and
the RL process with safe padding~\cite{HasanbeigAK20} (denoted as \SafePadding{}).

Learning of each controller consists of the \emph{training} and the \emph{test} phases.
The training phase is the main part of the learning, and the testing phase is invoked once every 10,000 training steps to evaluate the learned controller and choose the resulting one.
We finish the learning when the total number of steps in the training phase exceeds the predetermined bound, which is one of the commonly used criteria.
See \cref{table:summary_benchmarks} for the bounds for each benchmark.

To evaluate the RL training, we measure
 the total number of training episodes with undesired behaviors and
 the total execution time, including both training and testing phases.
To evaluate each controller, we run it for 30 episodes and measure 
 the \emph{mean reward} and the \emph{safe rate}, \ie{}
 the rate of the episodes without undesired behaviors in the 30 episodes.

We ran experiments on a GPU server with AMD EPYC 7702P, NVIDIA GeForce RTX 2080 Ti, 125GiB RAM, and Ubuntu 20.04.3 LTS.\@
We used eight CPUs and one GPU for each execution.
We ran each instance of the experiment 30 times, \ie{} we trained $7 \times 3 \times 30$ controllers in total. 
For each metric, we report the mean of the 30 executions, \ie{} we have $7 \times 3$ reported values.
%

\subsection{Benchmarks}\label{subsec:benchmarks}

\begin{table}[tbp]
 \caption{Summary of the benchmarks we used. MLP and CNN are abbreviations of ``multilayer perceptron'' and ``convolutional neural network''.}%
 \label{table:summary_benchmarks}
 \centering
 \scriptsize
 \begin{tabular}{l c c c c r}
  \toprule
  & Benchmark's origin &Observation space (size) & Network & Learning algorithm & \# of steps\\
  \midrule
  \WaterTank{}& Alshiekh et al.~\cite{AlshiekhBEKNT18} & Discrete (714) & MLP & PPO & 500,000\\
  \GridWorld{}& Our original & Discrete (625) & MLP & PPO & 100,000 \\
  \Taxi{}& OpenAI Gym~\cite{BrockmanCPSSTZ16}& Discrete (500) & MLP & PPO & 200,000\\
  \CliffWalking{}& OpenAI Gym~\cite{BrockmanCPSSTZ16}& Discrete (48) & MLP & PPO & 200,000\\
  \SelfDrivingCar{}& Alshiekh et al.~\cite{AlshiekhBEKNT18} & Continuous (${[-1,1]}^4$) & MLP & DQN\LongVersion{ with Boltzmann exploration} & 200,000\\
  \SideWalk{}& MiniWorld~\cite{gym_miniworld} & Image ($80 \times 60 \times 3 \times 256$) & CNN & PPO & 100,000\\
  \CarRacing{}& OpenAI Gym~\cite{BrockmanCPSSTZ16}& Image ($96 \times 96 \times 3 \times 256$) & CNN & PPO & 200,000\\
  \bottomrule
 \end{tabular}
\end{table}
We chose seven benchmarks for our experiments.
\cref{table:summary_benchmarks} summarizes them.
Most of them are common and openly available.
We modified them to fit to dynamic shielding: randomness except for those observable as \Env{} actions are removed; the actions are discretized; the observations for the RL agent is not changed, while the observation for the dynamic shield is discretized.
We used CNN for the benchmarks with graphical observation and MLP for the others.


\WaterTank{} implements the benchmark already used for shielding in~\cite{AlshiekhBEKNT18}.
\cref{example:water_tank:summary,example:water_tank:fsrs} show the details of the system.
The undesired behavior used in the shield construction is:
\begin{ienumeration}
 \item to make the water tank empty or full, or
 \item to change the status of the switch too often.
\end{ienumeration}

\GridWorld{} is a high-level robot control example of two robots in a $5 \times 5$ grid arena.
One of them is the ego robot we control, and the other robot randomly moves.
The objective of the ego robot is to reach the goal area, without touching the arena walls or the other robot.

\Taxi{} is a benchmark to pick up passengers and drop them off at another place in a $5 \times 5$ grid arena.
As criteria of a safety violation, we measured the number of episodes where the taxi is broken, which randomly occurs when it hits the wall.
The undesired behavior used in the shield construction is hitting the wall, picking up or dropping off the passenger at an inappropriate location.

\CliffWalking{} is a benchmark to move ego to the goal area without stepping off the cliff in a $3 \times 12$ grid arena.
The undesired behavior used in the shield construction is to step off the cliff.

\SelfDrivingCar{} is a benchmark to drive a car in a $480 \times 480$ arena around a blocked area in a clockwise direction without hitting the walls.
The undesired behavior used in the shield construction is to hit the walls in the arena.

\SideWalk{} is a benchmark for 3D robot simulation.
Its objective is to reach the goal area on the sidewalk without entering the roadway.
The undesired behavior used in the shield construction is to enter the roadway.

\CarRacing{} is a benchmark to drive a car on a predetermined racing course.
\LongVersion{The physical dynamics is simulated by Box2D, which is a 2D physics engine for games.}
As criteria of a safety violation, we measured the number of episodes with \emph{spin behavior}, \ie{} ego drives off the road and keeps rotating indefinitely.
The undesired behavior used in the shield construction is to deviate from the road for any consecutive steps, which is not yet a safety violation but tends to lead to it.

\begin{table}[tbp]
 \caption{Mean of 1) the number of episodes with undesired behavior in training phases and 2) the execution time (in seconds), including both training and testing phases. The cells with the best results are highlighted.}
 \label{table:summary_training_episodes}
 \centering
 \scriptsize
 \begin{tabular}{lrrrrrrccccccccccc}
  \toprule
  {} & \multicolumn{3}{c}{Undesired episodes} & \multicolumn{3}{c}{Total time} & \\
  {} & \multicolumn{1}{c}{\NoShield{}} & \multicolumn{1}{c}{\SafePadding{}} & \multicolumn{1}{c}{\DynamicShielding{}} & \multicolumn{1}{c}{\NoShield{}} & \multicolumn{1}{c}{\SafePadding{}} & \multicolumn{1}{c}{\DynamicShielding{}} \\
  \midrule
  \WaterTank{} & 1883.67 & 1892.40 & \tbcolor{} 177.13 & \tbcolor{} 1860.46 & 1947.09 & 6080.89 \\
  \GridWorld{} & 6996.40 & 7322.23 & \tbcolor{} 5623.43 & \tbcolor{} 177.18 & 1487.10 & 4548.70 \\
  \CliffWalking{} & 1493.20 & 1528.67 & \tbcolor{} 478.20 & \tbcolor{} 355.18 & 365.54 & 839.06 \\
  \Taxi{} & 8723.13 & 2057.33 & \tbcolor{} 37.77 & \tbcolor{} 336.04 & 349.78 & 611.87 \\
  \SelfDrivingCar{} & 6403.07 & 6454.60 & \tbcolor{} 5662.40 & \tbcolor{} 865.55 & 4919.13 & 10087.18 \\
  \SideWalk{} & 373.60 & 427.93 & \tbcolor{} 273.37 & \tbcolor{} 762.67 & 1734.66 & 6395.73 \\
  \CarRacing{} & 180.13 & 141.17 & \tbcolor{} 41.73 & \tbcolor{} 7650.38 & 16694.63 & 12532.04 \\
  \bottomrule
 \end{tabular}
\end{table}

\begin{figure}[tbp]
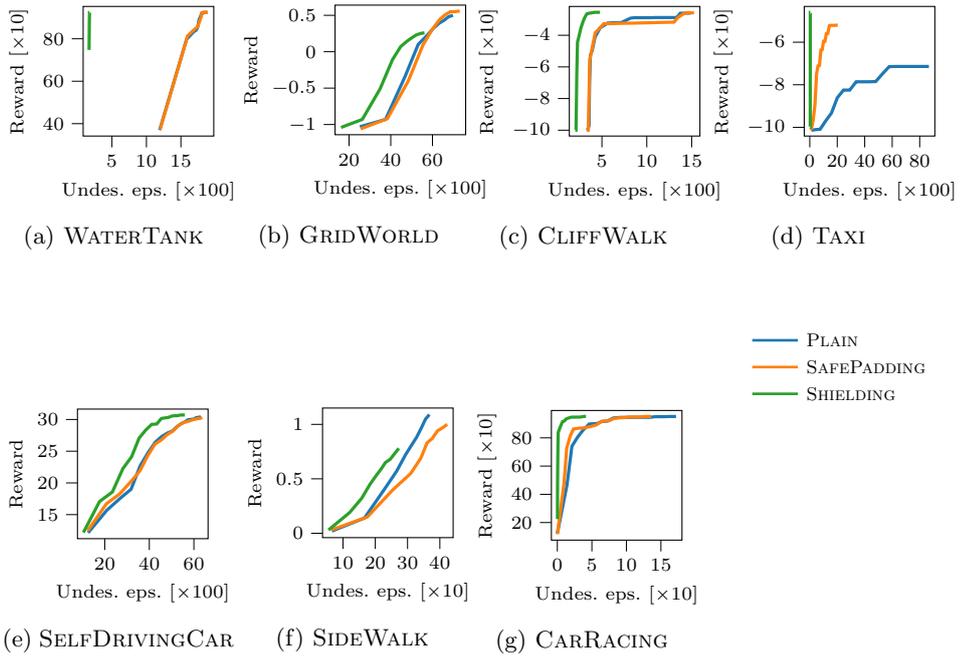

 \scriptsize
 \begin{subfigure}{.24\textwidth}
  \centering
  \input{./figs/crash_episodes_eval-best_mean_reward_Benchmarks.WATER_TANK.tikz}%
  \caption{\WaterTank{}}
 \end{subfigure}
 \hfill
 \begin{subfigure}{.24\textwidth}
  \centering
  \input{./figs/crash_episodes_eval-best_mean_reward_Benchmarks.GRID_WORLD.tikz}%
  \caption{\GridWorld{}}
 \end{subfigure}
 \hfill
 \begin{subfigure}{.24\textwidth}
  \centering
  \input{./figs/crash_episodes_eval-best_mean_reward_Benchmarks.CLIFFWALKING.tikz}%
  \caption{\CliffWalking{}}
 \end{subfigure}
 \hfill
 \begin{subfigure}{.24\textwidth}
  \centering
  \input{./figs/crash_episodes_eval-best_mean_reward_Benchmarks.TAXI.tikz}%
  \caption{\Taxi{}}
 \end{subfigure}
 \hfill
 \begin{subfigure}{.25\textwidth}
  \centering
  \input{./figs/crash_episodes_eval-best_mean_reward_Benchmarks.SELF_DRIVING_CAR.tikz}%
  \caption{\SelfDrivingCar{}}
 \end{subfigure}
 \hfill
 \begin{subfigure}{.24\textwidth}
  \centering
  \input{./figs/crash_episodes_eval-best_mean_reward_Benchmarks.SIDEWALK.tikz}%
  \caption{\SideWalk{}}
 \end{subfigure}
 \hfill
 \begin{subfigure}{.24\textwidth}
  \centering
  \input{./figs/crash_episodes_eval-best_mean_reward_Benchmarks.CAR_RACING.tikz}%
  \caption{\CarRacing{}}
 \end{subfigure}
 \begin{minipage}[t]{.03\textwidth} 
  \hfill
 \end{minipage}
 \begin{minipage}[t]{.21\textwidth}
  \centering
  \input{./figs/legend}%
 \end{minipage}
 \caption{The mean of the number of undesired episodes before each testing phase and the reward of the best controller obtained in the testing phase.}%
 \label{figure:summary_reward_vs_crashes}
\end{figure}

\subsection{RQ1: Safety by dynamic shielding in the training phase}\label{subsection:safety}

To evaluate the safety by dynamic shielding, we measured the number of training episodes with undesired behaviors (``Undesired episodes'' columns in \cref{table:summary_training_episodes}), and the relationship between the number of undesired training episodes before each testing phase and the highest mean reward by the testing phase (\cref{figure:summary_reward_vs_crashes}).

In \cref{table:summary_training_episodes}, we observe that, for all benchmarks, the training episodes with undesired behavior were, on average, the lowest when we used dynamic shielding.
For example, for \CarRacing{}, the mean number of the training episodes with undesired behavior was reduced by about 77\% compared to \NoShield{} and about 70\% compared to \SafePadding{}.
In \cref{figure:summary_reward_vs_crashes}, we observe that for all benchmarks, the curve of \DynamicShielding{} is growing faster than the curves of \NoShield{} and \SafePadding{}.
This suggests that dynamic shielding decreased the number of undesired explorations to obtain a controller with similar performance.

\rqanswer{RQ1}{Overall, we conclude that dynamic shielding can significantly reduce undesired behaviors during exploration.}

Compared to \NoShield{}, \DynamicShielding{} decreases such undesired training episodes because it prevents exploration of actions known to cause undesired behavior, while the plain RL may repeatedly explore such actions.
Although \SafePadding{} also prevents such an undesired exploration, the number of the undesired training episodes of \DynamicShielding{} was smaller than that of \SafePadding{}. 
This is because, with \DynamicShielding{}, the undesired explorations are generalized by state merging in the RPNI algorithm, while in \SafePadding{}, the state space of the learned MDP is identical to the observation space in the RL, and the undesired explorations are not often generalized.
\begin{LongVersionBlock}
Another reason is that in \SafePadding{}, a state is deemed to be safe only if the agent can be in the safe states after taking any action.
 For example, a state is deemed to be unsafe if there is a bad action causing an undesired behavior even if the other actions are totally harmless.
This definition of safe states is too pessimistic, and in many cases, \SafePadding{} cannot provide any safe actions, in which case the RL agent explores in the same way as \NoShield{}.
\end{LongVersionBlock}

\subsection{RQ2: Performance of the resulting controller}

To evaluate the performance of the resulting controller, for each instance of the experiments, we measured the mean reward and the safe rate using the controller that achieved the highest mean reward in the testing phases (\cref{table:summary_of_result_performance}).

In \cref{table:summary_of_result_performance}, we observe that dynamic shielding does not significantly decrease the safe rate for most of the benchmarks.
Moreover, the reward was the highest when we used dynamic shielding for most of the benchmarks.
This is likely because dynamic shielding prevents explorations with undesired behaviors and leads the learning agent to achieve the task.

In \cref{table:summary_of_result_performance}, we also observe that for \GridWorld{} and \SideWalk{}, the performance of the controllers obtained by \DynamicShielding{} was the worst.
This is likely because when a dynamic shield prevents undesired explorations, some of the useful explorations may also be prevented if they are deemed undesired when generalizing undesired behavior.
Note that such generalization is useful since it can prevent undesired explorations even if we have not experienced exactly the same exploration.
Nevertheless, \cref{table:summary_of_result_performance} also shows that the average degradation of the safe rate due to dynamic shielding is at most 12\% (\GridWorld{}).
This is not prohibitively large considering the reduction of undesired explorations (\eg{} from 6996.40 to 5623.43 for \GridWorld{} in \cref{table:summary_of_result_performance}).
%

\rqanswer{RQ2}{Overall, we conclude that dynamic shielding usually improves the performance of the resulting controller.
}

\begin{table}[tbp]
 \caption{Mean of the performance of the best controllers obtained in the 30 executions. The cells with the best results are highlighted.}%
 \label{table:summary_of_result_performance}
 \centering
 \scriptsize
  \begin{tabular}{lcccccccccccc}
  \toprule
  {} & \multicolumn{3}{c}{Mean reward} & \multicolumn{3}{c}{Safe rate} \\
  {} & \multicolumn{1}{c}{\NoShield{}} & \multicolumn{1}{c}{\SafePadding{}} & \multicolumn{1}{c}{\DynamicShielding{}} & \multicolumn{1}{c}{\NoShield{}} & \multicolumn{1}{c}{\SafePadding{}} & \multicolumn{1}{c}{\DynamicShielding{}} \\
  \midrule
  \WaterTank{} & 918.89 & 919.81 & \tbcolor{} 921.81 & \tbcolor{} 1.00 & \tbcolor{} 1.00 & \tbcolor{} 1.00 \\
  \GridWorld{} & 0.37 & \tbcolor{} 0.46 & 0.07 & 0.80 & \tbcolor{} 0.85 & 0.73 \\
  \CliffWalking{} & -69.13 & -66.00 & \tbcolor{} -65.93 & \tbcolor{} 1.00 & \tbcolor{} 1.00 & \tbcolor{} 1.00 \\
  \Taxi{} & -147.61 & -139.62 & \tbcolor{} -92.93 & 0.57 & 0.67 & \tbcolor{} 1.00 \\
  \SelfDrivingCar{} & 28.83 & 28.86 & \tbcolor{} 29.81 & \tbcolor{} 1.00 & \tbcolor{} 1.00 & \tbcolor{} 1.00 \\
  \SideWalk{} & \tbcolor{} 0.93 & 0.90 & 0.67 & \tbcolor{} 0.93 & 0.89 & 0.89 \\
  \CarRacing{} & 375.53 & 509.25 & \tbcolor{} 622.07 & \tbcolor{} 1.00 & \tbcolor{} 1.00 & \tbcolor{} 1.00 \\
  \bottomrule
 \end{tabular}
\end{table}

\subsection{RQ3: Time efficiency of dynamic shielding}\label{subsection:time_efficiency}
To evaluate the overhead of dynamic shielding, we measured the execution time, including both training and testing phases
(``Total time'' columns in
\cref{table:summary_training_episodes}).

In \cref{table:summary_training_episodes}, we observe that \NoShield{} was the fastest on average for all benchmarks.
This is mainly because of the computation cost to construct the set of safe actions.
We also observe that \SafePadding{} is usually faster than \DynamicShielding{}.
This is due to the combinatorial exploration in the RPNI algorithm to choose the merged states,
while \SafePadding{} does not conduct such state merging and  directly uses the observation space in the RL as the state space of the learned MDP.\@
We again remark that the generalization by the state merging contributes to the safety of \DynamicShielding{}.

In \cref{table:summary_training_episodes}, we also observe that the overhead of \DynamicShielding{} is at most about 2.5 hours (153.7 minutes in \SelfDrivingCar{}).
This is not negligible but still acceptable for many common usage scenarios, \eg{} learning a controller during non-working hours.
Moreover, \cref{table:summary_training_episodes} also shows that the overhead is not very sensitive to the cost of each execution
 and the observation space.
For example, the overhead of \DynamicShielding{} for \WaterTank{} was close to that for \CarRacing{}.
This indeed makes the relative overhead for \CarRacing{} small.

\rqanswer{RQ3}{Overall, we conclude that the overhead of dynamic shielding is not prohibitively large, especially when the RL is time-consuming even without shielding.}



\section{Conclusions and perspectives}\label{section:conclusion}

Based on passive automata learning, we proposed a new shielding scheme called \emph{dynamic shielding} for RL.\@
Since dynamic shielding does not require a predetermined system model, it is applicable to black-box systems.
The experiment results suggest that 
\begin{ienumeration}
 \item dynamic shielding prevents undesired exploration during training,
 \item dynamic shielding often improves the quality of the resulting controller, and
 \item the overhead of dynamic shielding is not prohibitively large.
\end{ienumeration}

In dynamic shielding, we assume that the system behaves deterministically once the actions of the controller and the environment are fixed.
Extending our approach, for example, utilizing a probabilistic model identification~\cite{DBLP:journals/corr/abs-1212-3873} to support probabilistic systems is a future direction.

\subsubsection*{Acknowledgements}

This work is partially supported by JST ERATO HASUO Metamathematics for Systems Design Project (No.\ JPMJER1603).
Masaki Waga is also supported by JST ACT-X Grant No.\ JPMJAX200U.
Stefan Klikovits is also supported by JSPS Grant-in-Aid\LongVersion{ for Research Activity Start-up} No.\ 20K23334.
Sasinee Pruekprasert is also supported by JSPS Grant-in-Aid\LongVersion{ for Early-Career Scientists} No.\ 21K14191.
Toru Takisaka is also supported by NSFC Research Fund for International Young Scientists No. 62150410437.

\ifdefined\VersionForArXiV%
	\bibliographystyle{alpha}
	\newcommand{\CCIS}{Communications in Computer and Information Science}
	\newcommand{\ENTCS}{Electronic Notes in Theoretical Computer Science}
	\newcommand{\FMSD}{Formal Methods in System Design}
	\newcommand{\IJFCS}{International Journal of Foundations of Computer Science}
	\newcommand{\IJSSE}{International Journal of Secure Software Engineering}
	\newcommand{\JLAP}{Journal of Logic and Algebraic Programming}
	\newcommand{\JLC}{Journal of Logic and Computation}
	\newcommand{\LMCS}{Logical Methods in Computer Science}
	\newcommand{\LNCS}{Lecture Notes in Computer Science}
	\newcommand{\RESS}{Reliability Engineering \& System Safety}
	\newcommand{\STTT}{International Journal on Software Tools for Technology Transfer}
	\newcommand{\TCS}{Theoretical Computer Science}
	\newcommand{\ToPNoC}{Transactions on Petri Nets and Other Models of Concurrency}
	\newcommand{\TSE}{IEEE Transactions on Software Engineering}
        \newcommand{\FGCS}{Future Generation Computing Systems}
        \newcommand{\SCP}{Science of Computer Programming}
        \newcommand{\IFAAMS}{International Foundation for Autonomous Agents and Multiagent Systems}
	\renewcommand*{\bibfont}{\footnotesize}
	\printbibliography[title={References}]
\else
	\newcommand{\CCIS}{CCIS}
	\newcommand{\ENTCS}{ENTCS}
	\newcommand{\FMSD}{FMSD}
	\newcommand{\IJFCS}{IJFCS}
	\newcommand{\IJSSE}{IJSSE}
	\newcommand{\JLAP}{JLAP}
	\newcommand{\JLC}{JLC}
	\newcommand{\LMCS}{LMCS}
	\newcommand{\LNCS}{LNCS}
	\newcommand{\RESS}{RESS}
	\newcommand{\STTT}{STTT}
	\newcommand{\TCS}{TCS}
	\newcommand{\ToPNoC}{ToPNoC}
	\newcommand{\TSE}{TSE}
        \newcommand{\FGCS}{Future Gener. Comput. Syst.}
          \newcommand{\SCP}{Science of Computer Programming}
        \newcommand{\IFAAMS}{IFAAMS}
        \bibliographystyle{splncs04}
\fi
\bibliography{AutomLearnShielding}

\ifdefined\VersionLong%
\clearpage
\appendix
\section{Post-posed shielding with online automata inference}\label{section:postposed_shielding}

Here, we show how to apply our dynamic shielding scheme for \emph{post-posed} shielding in~\cite[\S{}5.2]{AlshiekhBEKNT18}.
A post-posed shield is similar to a (preemptive) shield in \cref{def:preemptive_shield}, but instead of returning the set of safe actions of \Cont{},
a post-posed shield modifies an unsafe action of \Cont{} to enforce the given specification.

\begin{definition}[post-posed shield]
 Let  $\A = (\Loc, \locinit, \ActionsIn, \ActionsOut, \Edges, \Label)$ be an \ac{fsrs}.\@
 A \emph{post-posed shield} for $\A$ is a Mealy machine $\Shield = (\ShieldLoc, \shieldLocInit, \ActionsIn, \ActionsCtrl, \ShieldEdges, \ShieldLabel)$ \st{}
 for any $\shieldLoc \in \ShieldLoc$, \Cont{} action $\action^1 \in \ActionsCtrl$, and \Env{} actions $\actionIn^2, \actionIn'^2 \in \ActionsEnv$,
 we have $\ShieldLabel(\shieldLoc, (\action^1, \actionIn^2)) = \ShieldLabel(\shieldLoc, (\action^1, \actionIn'^2))$.
\end{definition}

For any input word $\wordIn \in \ActionsIn^*$, a post-posed shield $\Shield$ returns $\ShieldLabel(\wordIn)$ as a safe action for \Cont{} after $\wordIn$ is processed in $\A$.
A shield $\Shield$ canonically induces a strategy $\playerStrategy$ for an \ac{mdp} $\mdp^\envStrategy$ generated by any $\envStrategy$, namely a strategy $\playerStrategy$ such that 
$\playerStrategy(\hat\loc_0, a_1, \hat\loc_1, \ldots, a_k, \hat\loc_k)$ is the Dirac distribution such that $\playerStrategy(\hat\loc_0, a_1, \hat\loc_1, \ldots, a_k, \hat\loc_k)(\Label_\mathcal{S}(\wordIn)) = 1$.

The construction of a post-posed shield is almost the same as that of a preemptive shield.
We obtain a post-posed shield by taking one of the actions returned by a preemptive shield.
We note that to minimize the interference, it is preferred to have $\ShieldLabel(\shieldLoc, (\actionIn^1, \actionIn^2)) = \actionIn^1$ if $\actionIn^1$ is a safe action.

Since constructing a post-posed shield does not require any other information than that of a preemptive shield,
it is straightforward to construct a post-posed shield with online automata inference.

\section{Detail of the RPNI algorithm}\label{appendix:RPNI_detail}

The following shows the formal definition of \acp{ptmm}.

\begin{definition}
	[\Acf{ptmm}]\label{def:ptmm}
	For a finite training data $\trainingData \subseteq \ActionsIn^* \times \ActionsOut$,
	the \ac{ptmm} is the Mealy machine $\A = (\Loc, \locinit, \ActionsIn, \ActionsOut, \Edges, \Label)$ such that:
	\begin{myitemize}
		\item \LongVersion{$\Loc$ consists of the set of the input words in $\trainingData$, \ie{}} $\Loc = \{\word \mid \exists \actionOut \in \ActionsOut.\, (\word, \actionOut) \in \trainingData \}$;
		\item \LongVersion{the initial state $\locinit$ is} $\locinit = \varepsilon$;
		\item \LongVersion{the transition function} $\Edges$ and \LongVersion{the labeling function} $\Label$ are such that for each $\loc \in \Loc$ and $\actionPair \in \ActionsPair$, if $\loc \cdot \actionIn \in \Loc$ holds, $\Delta(\loc, \actionIn) = \loc \cdot \actionIn$ and $\Label(\loc,\actionIn) = \actionOut$, and undefined otherwise.
	\end{myitemize}
\end{definition}

\begin{algorithm}[bp]
	\caption{Outline of the RPNI algorithm}%
	\label{algorithm:RPNI}
	\DontPrintSemicolon{}
	\newcommand{\myCommentFont}[1]{\texttt{\footnotesize{#1}}}
	\SetCommentSty{myCommentFont}
	\SetKwFunction{FMerge}{merge}
	\Input{Training data $\trainingData \subseteq \ActionsIn^* \times \ActionsOut$}
	\Output{A Mealy machine $\A$ compatible with the training data $\trainingData$}
	$\A \gets \A_{\trainingData}$ \tcp*[r]{Initialize $\A$ by the \ac{ptmm} of $\trainingData$}\label{algorithm:RPNI:initializeA}
	$\red \gets \emptyset$; $\blue \gets \Loc$\;
\label{algorithm:RPNI:initializeRedBlue}
	\While{$\blue \neq \emptyset$} {\label{algorithm:RPNI:main_loop:begin}
		\KwPop{} the smallest element $\blueLoc$ \KwFrom{} $\blue$\label{algorithm:RPNI:popBlue}\;
		\If{$\exists \redLoc \in \red.\, \text{\FConsistent{$\redLoc, \blueLoc$}}$} {
			\KwMerge{} $\blueLoc$ \KwTo{} $\redLoc$\;
		} \Else{
			\KwPush{} $\blueLoc$ \KwTo{} $\red$\;\label{algorithm:RPNI:main_loop:end}
		}
	}
\end{algorithm}

\cref{algorithm:RPNI} outlines the \LongVersion{\ac{rpni}}\ShortVersion{\acs{rpni}} algorithm.
In the \LongVersion{\ac{rpni}}\ShortVersion{\acs{rpni}} algorithm, we have two sets, $\red$ and $\blue$, of the states of the intermediate Mealy machine  $\A$:
$\red$ is the set of the states in the resulting Mealy machine;
$\blue$ is the set of the states that may not be in the resulting Mealy machine.
In the loop from \crefrange{algorithm:RPNI:main_loop:begin}{algorithm:RPNI:main_loop:end}, for each $\blueLoc \in \blue$, we try to find $\redLoc \in \red$ such that we can merge $\blueLoc$ to $\redLoc$.
For the convergence\LongVersion{ of the \LongVersion{\ac{rpni}}\ShortVersion{\acs{rpni}} algorithm}, we try merging from the root to the leaves of the \ac{ptmm}.\@
Therefore, in \cref{algorithm:RPNI:popBlue}, we examine the minimum state in $\blue$ according to the following order ${\preceq}$:
we have $\loc \preceq \loc'$ if there is $\wordIn \in \ActionsIn^*$ satisfying $\locinit \xrightarrow{\wordIn} \loc$ and for any $\wordIn' \in \ActionsIn^*$ satisfying $\locinit \xrightarrow{\wordIn'} \loc'$, we have either $|\wordIn| < |\wordIn'|$ or $|\wordIn| = |\wordIn'|$ and $\wordIn$ is smaller than $\wordIn'$ in the lexicographical order. 

\section{Detail of our definition of $\mindepth$}

\newcommand{\maxEpisodeLength}{\mathrm{MaxEpLen}}
\newcommand{\maxMinDepth}{\mindepth_\mathrm{max}}

Here, we show how we decide $\mindepth$ in our implementation.
Let $\runs$ be the set of the runs we have experienced before we evaluate $\mindepth$, $\maxEpisodeLength$ be the maximum length of each episode, and $\maxMinDepth$ be the maximum bound of $\mindepth$. The use of $\maxMinDepth$ is to prevent $\mindepth$ from being very large at the beginning of the learning. In our experiments, we used $\maxMinDepth = 5$.
\cref{equation:min_depth} shows the definition of $\mindepth$ in our implementation, where $\lceil A \rceil$ is the ceiling function.

\begin{equation}
 \mindepth = \min\left(|\runs| \frac{\lceil \maxEpisodeLength - \frac{1}{|\runs|} \sum_{\run \in \runs} |\run| \rceil}{\sum_{\run \in \runs} |\run|}, \maxMinDepth\right) 
  \label{equation:min_depth}
\end{equation}

As we discussed in \cref{subsection:min_depth}, $\mindepth{}$ is large when the mean length of the episodes is short.

\section{Detail of the benchmarks}%
\label{appendix:detail_benchmarks}

Here we show the detail of the benchmarks used in our experiments.

\subsection{\WaterTank}

\WaterTank{} is the benchmark described in \cref{example:water_tank:summary,example:water_tank:fsrs}. The input space is open and close. The output space is \{open, close\} $\times$ \{low, safe, high\} $\times$ \{valve violation, no violation\}. As explained in \cref{example:water_tank:summary}, the source of nondeterminism is the inflow and outflow of water. In this example, the undesired behavior is either having a switch violation or making the tank full or empty. 

\subsection{\GridWorld}

\begin{figure}[tbp]
\begin{minipage}[t]{0.3\textwidth}
\end{minipage}
\hfill
\begin{minipage}[t]{0.12\textwidth}
 \begin{verbatim}
xxxxxxx
x    Ex
x xxx x
x xG  x
x xxx x
xS    x
xxxxxxx
 \end{verbatim}
\end{minipage}
\hfill
\begin{minipage}[t]{0.3\textwidth}
\end{minipage}
\caption{The arena of \GridWorld{}. ``x'' represents the wall, ``S'' represents the ego robot's starting point, ``G'' represents the goal, and ``E'' represents the other robot's starting point.}%
 \label{figure:screenshot_grid_world}
\end{figure}
\GridWorld{} is a high-level robot control example in a $5 \times 5$ grid arena.
Two robots are in the arena; one of them is the ego robot we control, and the other robot randomly moves.
We note that the move of the other robot is encoded by actions of \Env{}, and the system is still deterministic in the sense that it is represented by an \ac{fsrs}.
The objective of the ego robot is to reach the goal area without hitting either the wall or the other robot.

The propositions used for shielding are as follows.
\begin{itemize}
 \item Either the ego robot hits the wall or not
 \item Either the ego robot hits the other robot or not
\end{itemize}
The undesired behavior used in the shield construction is to hit either the wall or the other robot.

\subsection{\Taxi}

\Taxi{}\cite{dietterich2000hierarchical} is a $5 \times 5$ grid map in which a taxi has to pick up a passenger from a location and drop it at a different one. The map has some barriers that the taxi can not cross (see \cref{figure:screenshot_taxi}). In our extended version of the problem, the taxi can break nondeterministically when hitting a barrier. Every time the taxi hits a barrier, it gets damaged. Each of these damages increases the likelihood of breaking the car at the next hit to a barrier. 
 Although the taxi nondeterministically becomes broken, our assumption on the system determinism still holds by letting the \Env{}'s action be ``Break'' when the taxi becomes broken.

\begin{figure}[tbp]
	\begin{minipage}[t]{0.3\textwidth}
	\end{minipage}
	\hfill
	\begin{minipage}[t]{0.12\textwidth}
		\begin{verbatim}
		+---------+
		|R: | : :G|
		| : | : : |
		| : :T: : |
		| | : | : |
		|Y| : |B: |
		+---------+
		\end{verbatim}
	\end{minipage}
	\hfill
	\begin{minipage}[t]{0.3\textwidth}
	\end{minipage}
	\caption{The arena of \Taxi{}. ``T'' represents the taxi initial position, whereas ``R'',``G'',``Y'',``B'' represent possible pick-up and drop-off locations for the passenger. The taxi can cross cells separated by ``:'' but it is not allowed to cross ``\textbar''. }%
	\label{figure:screenshot_taxi}
\end{figure}

The input space is the actions related to the movement of the taxi (south, north, east, west) and the actions related to the passenger (pick-up or drop-off). The taxi can observe if it is a successful movement or if it hits a barrier. It can also observe when it arrives at a drop-off/pick-up location, when it does a successful or wrong pick-up/drop-off, and when the passenger is inside the taxi.
It can also observe the position in the arena.

In this example, the undesired behavior can be produced by scratching the taxi, breaking the taxi, or doing wrong pick-up/drop-off, because all these actions can damage either the vehicle or the passenger.

\subsection{\CliffWalking}

\CliffWalking{}\cite{SuttonB98} is a $3 \times 13$ grid map in which the agent needs to go from the initial position to the final position without falling into a cliff (see \cref{figure:screenshot_cliff_walking}). The agent can move in four directions (south, north, east, and west), and it can only observe if it is a safe move, a fall into the cliff, or reached the goal. In this simple problem, the only source of the undesired behavior is falling into the cliff.

\begin{figure}[tbp]
	\begin{minipage}[t]{1.5\textwidth}
	\end{minipage}
	\hspace{32mm}
	\begin{minipage}[t]{0.13\textwidth}
		\begin{verbatim}
		+-------------------------+
		| : : : : : : : : : : : : |
		| : : : : : : : : : : : : |
		|S:x:x:x:x:x:x:x:x:x:x:x:G|
		+-------------------------+
		\end{verbatim}
	\end{minipage}
	\hfill
	\begin{minipage}[t]{0.8\textwidth}
	\end{minipage}
	\caption{The arena of \CliffWalking{}. ``S'' represents the initial position, and ``G'' represents the target position. The cliff is denoted with ``x''. }%
	\label{figure:screenshot_cliff_walking}
\end{figure}

\subsection{\SelfDrivingCar}
\begin{figure}[tbp]
 \centering
 \includegraphics[width=.6\linewidth,angle=0]{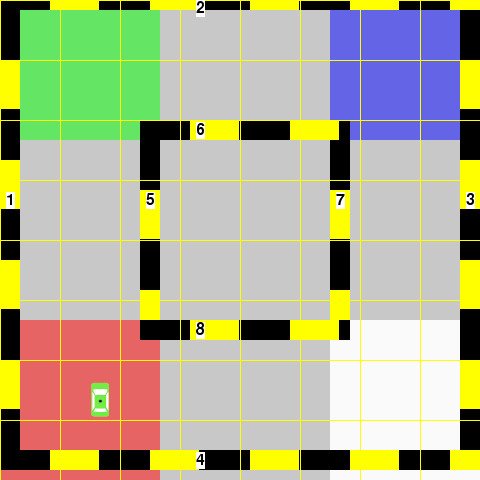}
 \caption{Screenshot of \SelfDrivingCar{}. The yellow-black stripes show the walls.\label{figure:screenshot_self_driving_car}}
 \end{figure}

\SelfDrivingCar{} benchmark is modified from a benchmark in~\cite{AlshiekhBEKNT18}.
The objective is to train a car in a $480 \times 480$ pixels arena to run in a clockwise direction without hitting the walls (see~\cref{figure:screenshot_self_driving_car}). 
The state of the car consists of its $xy$-position and heading angle. The car starts from $xy$-position $(100, 80)$ heading upward (heading angle = $\pi/2$).

A training step consists of 10 time steps in the simulation.
At each training step, the car chooses its action to be ``go-straight'', ``turn-left'', or ``turn-right''.
Based on its action, the car moves with the speed of 3 pixels per simulation time step, either following the same heading direction (go-straight) or changing the heading angle (turn-left or turn-right) by $\pi/40$ degrees per simulation time step.
For example, the car moves 30 pixels following its heading direction during a training step if it chooses ``go-straight''.

As in~\cite{AlshiekhBEKNT18},
the car is trained using a Deep Q-Network (DQN) with a Boltzmann
exploration policy. 
In each training step, the car
obtains a positive reward if it
 moves in a clockwise direction (determined using the start and end positions of the learning step), and a penalty
otherwise. 
A training episode ends if the car hits a wall or reaches 45 training steps. 
The training is successful if the car can visit green, blue, white, and red areas in order.

The undesired behavior used in the shield construction is to hit a wall.  
The input to the shield contains the information of $(\mathit{pos}, \mathit{dir}, \mathit{safe})$, where $\mathit{pos}$ is the abstract $xy$-position using $60 \times 60$ pixels grids, $\mathit{dir}$ is the abstract heading angle using the intervals in
$\{ \big[d\pi/4 - \pi/8, (d+1)\pi/4- \pi/8 \big) \mid d = 0,1, \ldots, 7 \}$,
 and $ \mathit{safe}$ indicates whether the car hits a wall during the previous training step.

\subsection{\SideWalk}

\begin{figure}[tbp]
 \centering
 \includegraphics[width=.3\linewidth,page=5,angle=270]{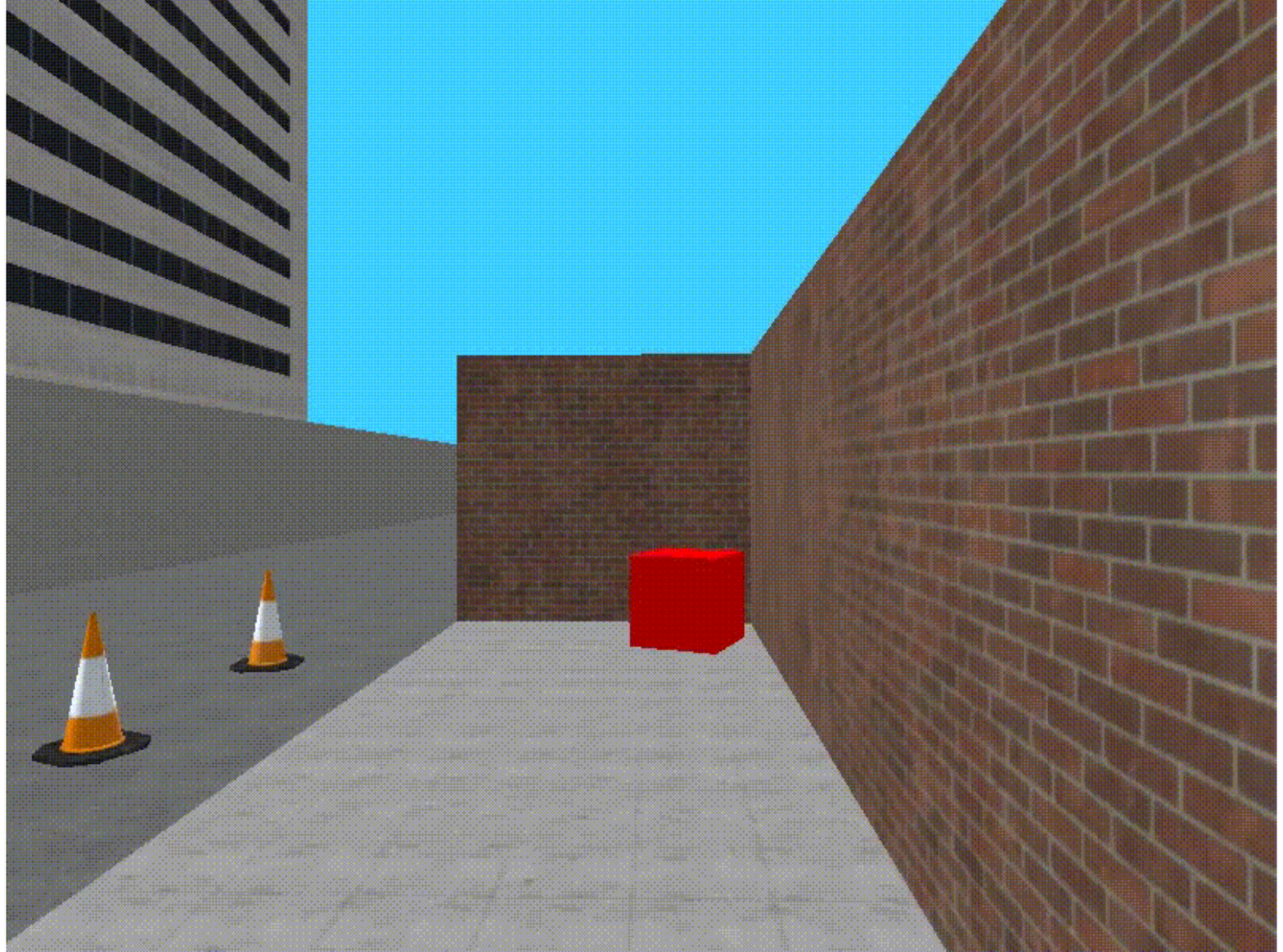}
 \caption{Screenshot of \SideWalk{}. The red box shows the goal area.}%
 \label{figure:screenshot_sidewalk}
\end{figure}
\SideWalk{} is a variant of a benchmark in MiniWorld~\cite{gym_miniworld} on 3D robot simulation.
\cref{figure:screenshot_sidewalk} shows a screenshot of \SideWalk{}.
Its objective is to reach the goal area in the sidewalk without entering the roadway.

The following summarizes the modifications from the original benchmark.

\begin{itemize}
 \item The random seed is chosen from 10 possible values.
 \item The domain randomization is enabled.
\end{itemize}

The propositions used for shielding are as follows.
\begin{itemize}
 \item The position of the robot is
       \begin{ienumeration}
        \item in the roadway,
        \item in front of the wall, or
        \item the others.
       \end{ienumeration}
 \item The robot observes the cone 
       \begin{ienumeration}
        \item in the left of the image,
        \item in the right of the image,
        \item in both left and right of the image, or
        \item nowhere in the image.
       \end{ienumeration}
 \item The robot observes the goal area or not.
 \item The robot observes the wall in the center of the image or not.
\end{itemize}
These propositions on the robot's perception are determined by the colors of certain pixels of the observed image.
The undesired behavior used in the shield construction is to enter the roadway.

\subsection{\CarRacing}

\begin{figure}[tbp]
 \centering
 \includegraphics[width=.5\linewidth,page=5,angle=270]{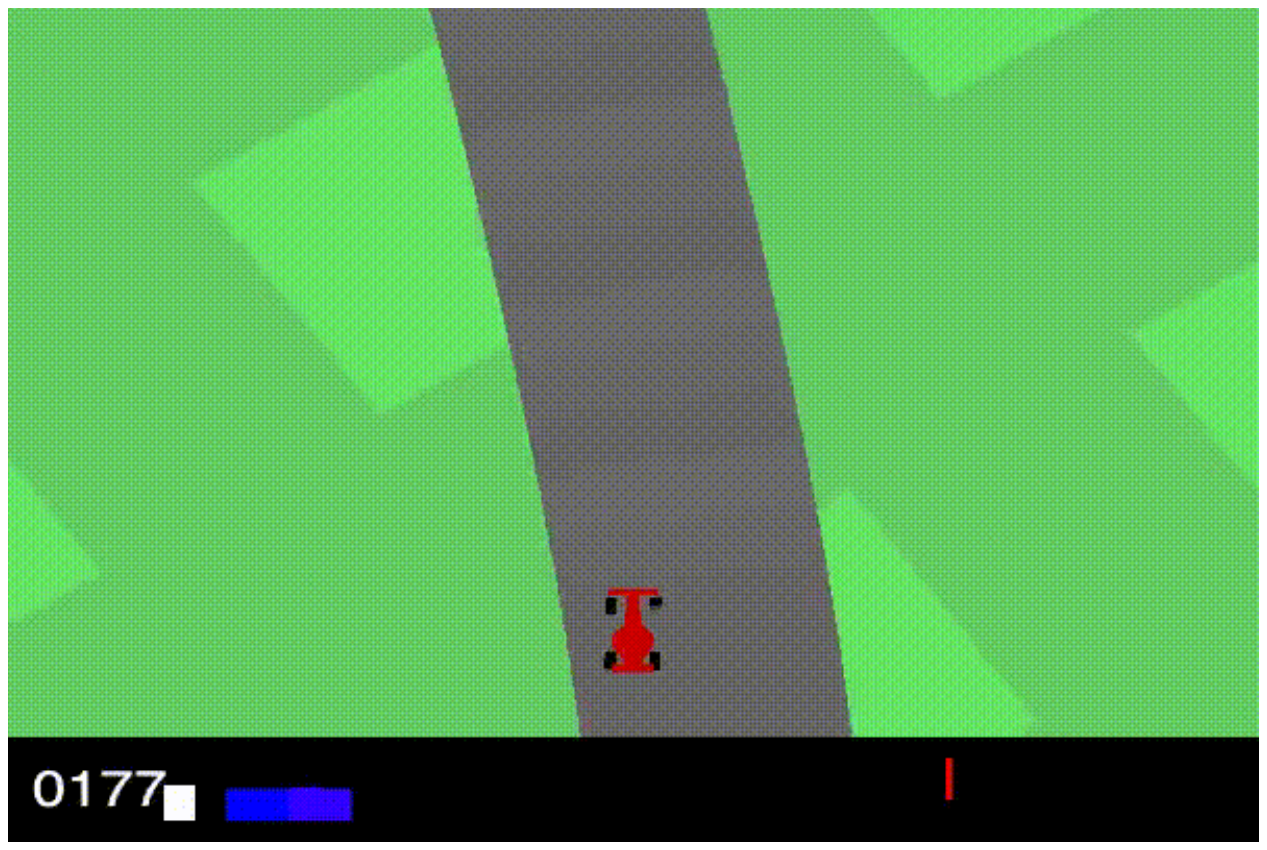}
 \caption{Screenshot of \CarRacing{}}%
 \label{figure:screenshot_car_racing}
\end{figure}
\CarRacing{} is a variant of a benchmark in OpenAI Gym~\cite{BrockmanCPSSTZ16}.
\cref{figure:screenshot_car_racing} shows a screenshot of \CarRacing{}.
Its objective is to drive a car on a predetermined racing course.
The physical dynamics are simulated by Box2D\footnote{Box2D is available at \url{https://box2d.org/}}, a 2D physics engine for games.

The following summarizes the modifications from the original benchmark.
\begin{itemize}
 \item The input space is discretized into ``NONE'',  ``ACCEL'',  ``RIGHT'',  ``LEFT'', and ``BRAKE''.
 \item The course creation is determinized by fixing the seed of the random number generator.
 \item The parameter of the course generation is modified so that the curves are smoother.
\end{itemize}

The propositions used for shielding are as follows.
\begin{itemize}
 \item The position of the car is
       \begin{ienumeration}
        \item on the road,
        \item on the grass area, or
        \item out of the arena, \ie{} crashes to the wall.
       \end{ienumeration}
 \item The car observes that 
       \begin{ienumeration}
        \item the road curves to the left,
        \item the road curves to the right, or
        \item the road is straight.
       \end{ienumeration}
\end{itemize}
These propositions on the car's perception are determined from the colors of certain pixels of the observed image.
As criteria of a safety violation, we measured the number of episodes with \emph{spin behavior}, \ie{} where ego drives off the road and keeps rotating indefinitely.
The undesired behavior used in the shield construction is to deviate from the road for any consecutive steps, which is not a safety violation but tends to be a safety violation.

\section{Detailed hyperparameters used in our experiments}%
\label{appendix:detail_of_learning}

\begin{table*}[tbp]
 \footnotesize
 \centering
 \caption{Detailed learning hyperparameters used in our experiments}%
 \label{table:detail_of_learning}
 \begin{tabular}[t]{lccccc}
  \toprule
  & Learning rate & Gamma & Batch size & \texttt{nb\_steps\_warmup} & \texttt{target\_model\_update}\\
  \midrule
  \WaterTank{} & $1 \times 10^{-4}$  & 0.99 & 64 & N/A & N/A\\
  \GridWorld{} & $3 \times 10^{-4}$  & 0.99 & 64 & N/A & N/A\\
  \CliffWalking{} & $1 \times 10^{-3}$ & 0.99 & 64 & N/A & N/A\\
  \Taxi{} & $1 \times 10^{-3}$ & 0.99 & 64 & N/A & N/A\\
  \SelfDrivingCar{} & $3 \times 10^{-4}$ & 0.99 & 32 & 1,000 & 10,000\\
  \SideWalk{} & $5 \times 10^{-5}$ & 0.99 & 64 & N/A & N/A\\
  \CarRacing{} & $3 \times 10^{-4}$ & 0.99 & 64 & N/A & N/A\\
  \bottomrule
 \end{tabular}
\end{table*}

\cref{table:detail_of_learning} summarizes the detailed hyperparameter of our experiment.
We tuned these hyperparameters by a grid search when the default parameters defined by the libraries do not work well for the learning without shielding.
For the hyperparameters not in \cref{table:detail_of_learning}, we used the default parameters defined by the libraries.

\section{Detailed experiment results and discussion}\label{appendix:detail_of_result}

Here, we show some detailed experiment results and discussion.

\subsection{Efficiency of learning with dynamic shielding}

\begin{figure}[tbp]
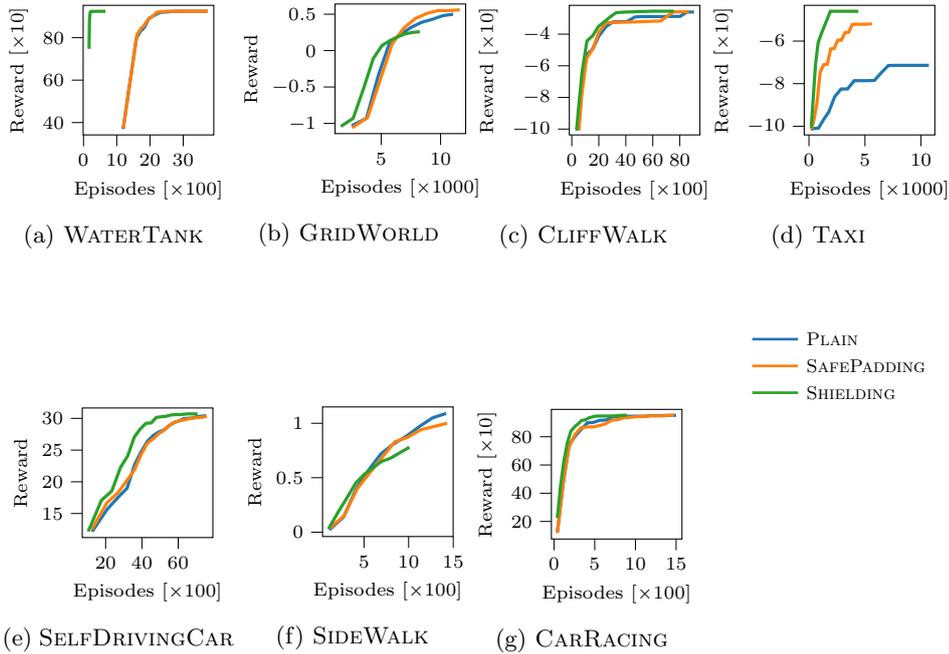

 \scriptsize
 \begin{subfigure}{.24\textwidth}
  \centering
  \input{./figs/episodes_eval-best_mean_reward_Benchmarks.WATER_TANK.tikz}%
  \caption{\WaterTank{}}
 \end{subfigure}
 \hfill
 \begin{subfigure}{.24\textwidth}
  \centering
  \input{./figs/episodes_eval-best_mean_reward_Benchmarks.GRID_WORLD.tikz}%
  \caption{\GridWorld{}}
 \end{subfigure}
 \hfill
 \begin{subfigure}{.24\textwidth}
  \centering
  \input{./figs/episodes_eval-best_mean_reward_Benchmarks.CLIFFWALKING.tikz}%
  \caption{\CliffWalking{}}
 \end{subfigure}
 \hfill
 \begin{subfigure}{.24\textwidth}
  \centering
  \input{./figs/episodes_eval-best_mean_reward_Benchmarks.TAXI.tikz}%
  \caption{\Taxi{}}
 \end{subfigure}
 \hfill
 \begin{subfigure}{.25\textwidth}
  \centering
  \input{./figs/episodes_eval-best_mean_reward_Benchmarks.SELF_DRIVING_CAR.tikz}%
  \caption{\SelfDrivingCar{}}
 \end{subfigure}
 \hfill
 \begin{subfigure}{.24\textwidth}
  \centering
  \input{./figs/episodes_eval-best_mean_reward_Benchmarks.SIDEWALK.tikz}%
  \caption{\SideWalk{}}
 \end{subfigure}
 \hfill
 \begin{subfigure}{.24\textwidth}
  \centering
  \input{./figs/episodes_eval-best_mean_reward_Benchmarks.CAR_RACING.tikz}%
  \caption{\CarRacing{}}
 \end{subfigure}
 \begin{minipage}[t]{.03\textwidth} 
  \hfill
 \end{minipage}
 \begin{minipage}[t]{.21\textwidth}
  \centering
  \input{./figs/legend}%
 \end{minipage}
 \caption{The mean of the number of episodes before each testing phase and the reward of the best controller obtained by the testing phase.}%
 \label{figure:summary_reward_vs_episodes}
\end{figure}

\begin{figure}[tbp]
 \scriptsize
 \begin{subfigure}{.24\textwidth}
  \centering
  \input{./figs/steps_eval-best_mean_reward_Benchmarks.WATER_TANK.tikz}%
  \caption{\WaterTank{}}
 \end{subfigure}
 \hfill
 \begin{subfigure}{.24\textwidth}
  \centering
  \input{./figs/steps_eval-best_mean_reward_Benchmarks.GRID_WORLD.tikz}%
  \caption{\GridWorld{}}
 \end{subfigure}
 \hfill
 \begin{subfigure}{.24\textwidth}
  \centering
  \input{./figs/steps_eval-best_mean_reward_Benchmarks.CLIFFWALKING.tikz}%
  \caption{\CliffWalking{}}
 \end{subfigure}
 \hfill
 \begin{subfigure}{.24\textwidth}
  \centering
  \input{./figs/steps_eval-best_mean_reward_Benchmarks.TAXI.tikz}%
  \caption{\Taxi{}}
 \end{subfigure}
 \hfill
 \begin{subfigure}{.25\textwidth}
  \centering
  \input{./figs/steps_eval-best_mean_reward_Benchmarks.SELF_DRIVING_CAR.tikz}%
  \caption{\SelfDrivingCar{}}
 \end{subfigure}
 \hfill
 \begin{subfigure}{.24\textwidth}
  \centering
  \input{./figs/steps_eval-best_mean_reward_Benchmarks.SIDEWALK.tikz}%
  \caption{\SideWalk{}}
 \end{subfigure}
 \hfill
 \begin{subfigure}{.24\textwidth}
  \centering
  \input{./figs/steps_eval-best_mean_reward_Benchmarks.CAR_RACING.tikz}%
  \caption{\CarRacing{}}
 \end{subfigure}
 \begin{minipage}[t]{.03\textwidth} 
  \hfill
 \end{minipage}
 \begin{minipage}[t]{.21\textwidth}
  \centering
  \input{./figs/legend}%
 \end{minipage}
 \caption{The mean of the number of steps before each testing phase and the reward of the best controller obtained by the testing phase.}%
 \label{figure:summary_reward_vs_steps}
\end{figure}

\begin{figure}[tbp]
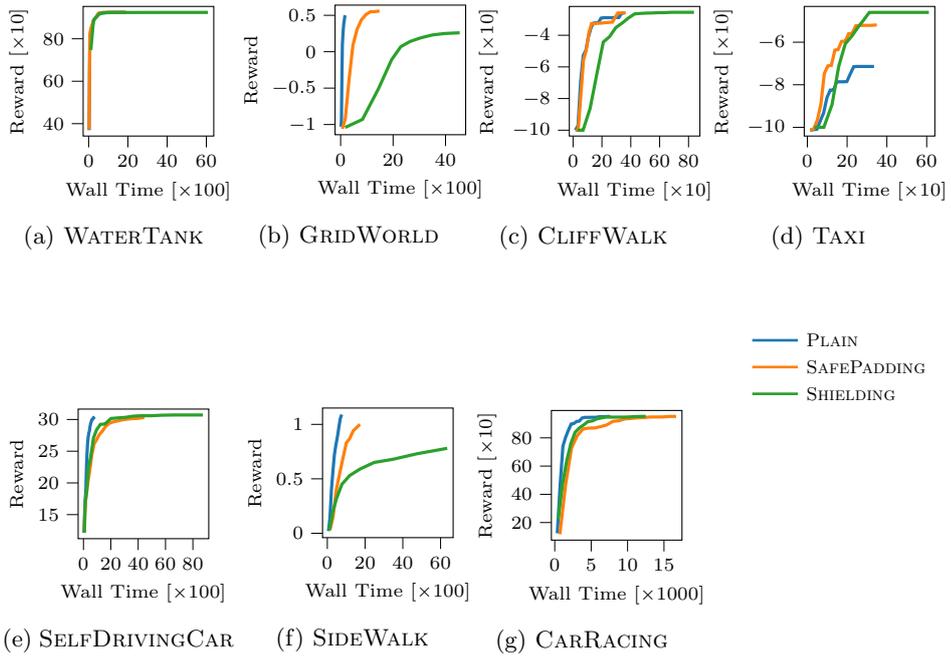

 \scriptsize
 \begin{subfigure}{.24\textwidth}
  \centering
  \input{./figs/eval-wall_time_eval-best_mean_reward_Benchmarks.WATER_TANK.tikz}%
  \caption{\WaterTank{}}
 \end{subfigure}
 \hfill
 \begin{subfigure}{.24\textwidth}
  \centering
  \input{./figs/eval-wall_time_eval-best_mean_reward_Benchmarks.GRID_WORLD.tikz}%
  \caption{\GridWorld{}}
 \end{subfigure}
 \hfill
 \begin{subfigure}{.24\textwidth}
  \centering
  \input{./figs/eval-wall_time_eval-best_mean_reward_Benchmarks.CLIFFWALKING.tikz}%
  \caption{\CliffWalking{}}
 \end{subfigure}
 \hfill
 \begin{subfigure}{.24\textwidth}
  \centering
  \input{./figs/eval-wall_time_eval-best_mean_reward_Benchmarks.TAXI.tikz}%
  \caption{\Taxi{}}
 \end{subfigure}
 \hfill
 \begin{subfigure}{.25\textwidth}
  \centering
  \input{./figs/eval-wall_time_eval-best_mean_reward_Benchmarks.SELF_DRIVING_CAR.tikz}%
  \caption{\SelfDrivingCar{}}
 \end{subfigure}
 \hfill
 \begin{subfigure}{.24\textwidth}
  \centering
  \input{./figs/eval-wall_time_eval-best_mean_reward_Benchmarks.SIDEWALK.tikz}%
  \caption{\SideWalk{}}
 \end{subfigure}
 \hfill
 \begin{subfigure}{.24\textwidth}
  \centering
  \input{./figs/eval-wall_time_eval-best_mean_reward_Benchmarks.CAR_RACING.tikz}%
  \caption{\CarRacing{}}
 \end{subfigure}
 \begin{minipage}[t]{.03\textwidth} 
  \hfill
 \end{minipage}
 \begin{minipage}[t]{.21\textwidth}
  \centering
  \input{./figs/legend}%
 \end{minipage}
 \caption{The mean of the total execution time before each testing phase and the reward of the best controller obtained by the testing phase.}%
 \label{figure:summary_reward_vs_execution_time}
\end{figure}

\cref{figure:summary_reward_vs_episodes,figure:summary_reward_vs_steps,figure:summary_reward_vs_execution_time} show the mean of the reward of the best controllers obtained by each testing phase and the number of episodes, the number of steps, and the total execution time before the testing phase, respectively.

In \cref{figure:summary_reward_vs_episodes}, we observe that \DynamicShielding{} usually requires a similar number of episodes to obtain a controller with similar performance compared to \NoShield{} or \SafePadding{}.
We also observe that \DynamicShielding{} sometimes requires fewer episodes to obtain a controller with similar performance than \NoShield{} or \SafePadding{} (\eg{} \WaterTank{} and \Taxi{}).
This improvement is likely because \DynamicShielding{} successfully prevented useless exploration.

\begin{figure}[tbp]
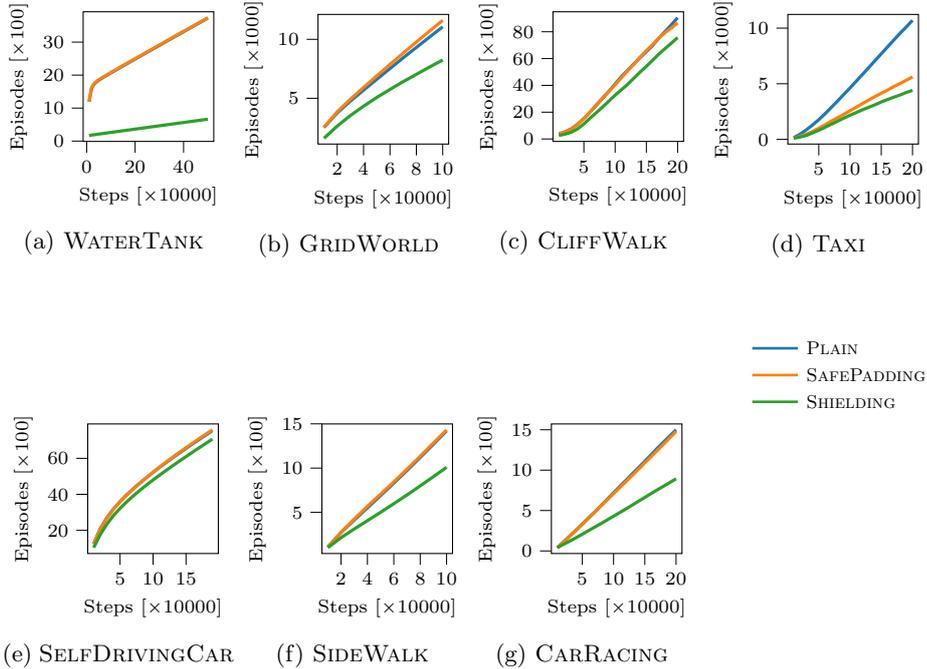

 \scriptsize
 \begin{subfigure}{.24\textwidth}
  \centering
  \input{./figs/steps_episodes_Benchmarks.WATER_TANK.tikz}%
  \caption{\WaterTank{}}
 \end{subfigure}
 \hfill
 \begin{subfigure}{.24\textwidth}
  \centering
  \input{./figs/steps_episodes_Benchmarks.GRID_WORLD.tikz}%
  \caption{\GridWorld{}}
 \end{subfigure}
 \hfill
 \begin{subfigure}{.24\textwidth}
  \centering
  \input{./figs/steps_episodes_Benchmarks.CLIFFWALKING.tikz}%
  \caption{\CliffWalking{}}
 \end{subfigure}
 \hfill
 \begin{subfigure}{.24\textwidth}
  \centering
  \input{./figs/steps_episodes_Benchmarks.TAXI.tikz}%
  \caption{\Taxi{}}
 \end{subfigure}
 \hfill
 \begin{subfigure}{.25\textwidth}
  \centering
  \input{./figs/steps_episodes_Benchmarks.SELF_DRIVING_CAR.tikz}%
  \caption{\SelfDrivingCar{}}
 \end{subfigure}
 \hfill
 \begin{subfigure}{.24\textwidth}
  \centering
  \input{./figs/steps_episodes_Benchmarks.SIDEWALK.tikz}%
  \caption{\SideWalk{}}
 \end{subfigure}
 \hfill
 \begin{subfigure}{.24\textwidth}
  \centering
  \input{./figs/steps_episodes_Benchmarks.CAR_RACING.tikz}%
  \caption{\CarRacing{}}
 \end{subfigure}
 \begin{minipage}[t]{.03\textwidth} 
  \hfill
 \end{minipage}
 \begin{minipage}[t]{.21\textwidth}
  \centering
  \input{./figs/legend}%
 \end{minipage}
 \caption{The mean of the number of training steps and the number of training episodes before each testing phase.}%
 \label{figure:episodes_vs_steps}
\end{figure}
In \cref{figure:summary_reward_vs_steps}, we observe that \DynamicShielding{} usually requires a similar number of steps to obtain a controller with similar performance compared with \NoShield{} or \SafePadding{}.
We also observe that \DynamicShielding{} sometimes requires more episodes to obtain a controller with similar performance than \NoShield{} or \SafePadding{} (\eg{} \GridWorld{} and \SideWalk{}).
This is likely because \DynamicShielding{} tends to make each episode longer by preventing failures, and the number of steps to experience a similar number of episodes tends to be larger. We can observe this tendency also in \cref{figure:episodes_vs_steps}, which shows the relationship between the number of training episodes and the training steps before each testing phase.

In \cref{figure:summary_reward_vs_execution_time}, we observe that \DynamicShielding{} usually takes a much longer time to obtain a controller with similar performance compared with \NoShield{} or \SafePadding{}.
As we discussed in \cref{subsection:time_efficiency} this is mainly because of the overhead of shield construction.

\subsection{Safety of learning with dynamic shielding}

\begin{figure}[tbp]
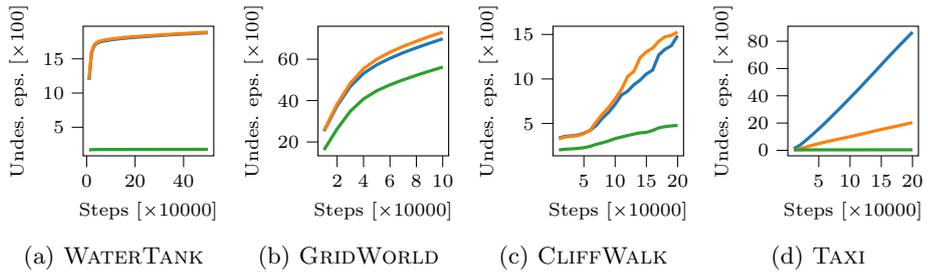
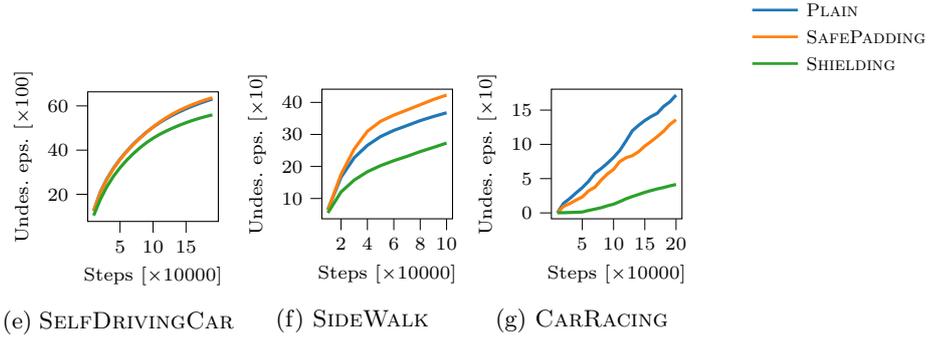

 \scriptsize
 \begin{subfigure}{.24\textwidth}
  \centering
  \input{./figs/steps_crash_episodes_Benchmarks.WATER_TANK.tikz}%
  \caption{\WaterTank{}}
 \end{subfigure}
 \hfill
 \begin{subfigure}{.24\textwidth}
  \centering
  \input{./figs/steps_crash_episodes_Benchmarks.GRID_WORLD.tikz}%
  \caption{\GridWorld{}}
 \end{subfigure}
 \hfill
 \begin{subfigure}{.24\textwidth}
  \centering
  \input{./figs/steps_crash_episodes_Benchmarks.CLIFFWALKING.tikz}%
  \caption{\CliffWalking{}}
 \end{subfigure}
 \hfill
 \begin{subfigure}{.24\textwidth}
  \centering
  \input{./figs/steps_crash_episodes_Benchmarks.TAXI.tikz}%
  \caption{\Taxi{}}\label{figure:crashes_vs_steps:taxi}
 \end{subfigure}
 \hfill
 \begin{subfigure}{.25\textwidth}
  \centering
  \input{./figs/steps_crash_episodes_Benchmarks.SELF_DRIVING_CAR.tikz}%
  \caption{\SelfDrivingCar{}}
 \end{subfigure}
 \hfill
 \begin{subfigure}{.24\textwidth}
  \centering
  \input{./figs/steps_crash_episodes_Benchmarks.SIDEWALK.tikz}%
  \caption{\SideWalk{}}
 \end{subfigure}
 \hfill
 \begin{subfigure}{.24\textwidth}
  \centering
  \input{./figs/steps_crash_episodes_Benchmarks.CAR_RACING.tikz}%
  \caption{\CarRacing{}}
 \end{subfigure}
 \begin{minipage}[t]{.03\textwidth} 
  \hfill
 \end{minipage}
 \begin{minipage}[t]{.21\textwidth}
  \centering
  \input{./figs/legend}%
 \end{minipage}
 \caption{The mean of the number of undesired training episodes and the number of steps before each testing phase}%
 \label{figure:crashes_vs_steps}
\end{figure}

\begin{figure}[tbp]
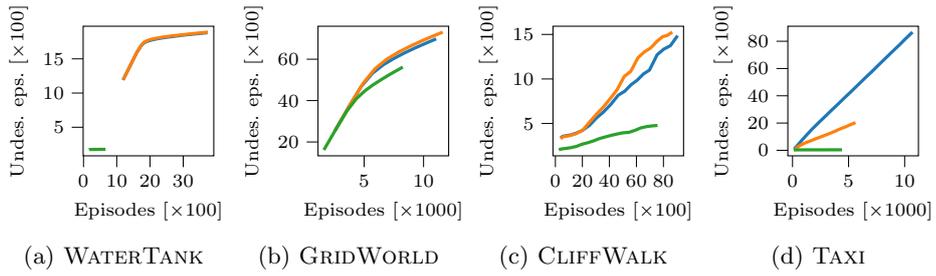
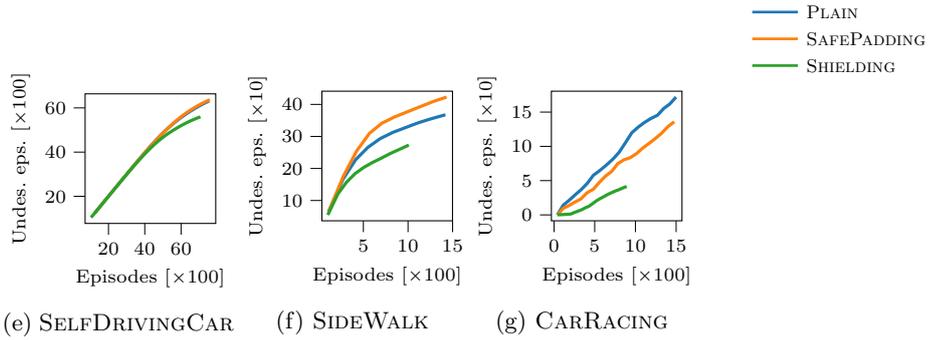

 \scriptsize
 \begin{subfigure}{.24\textwidth}
  \centering
  \input{./figs/episodes_crash_episodes_Benchmarks.WATER_TANK.tikz}%
  \caption{\WaterTank{}}
 \end{subfigure}
 \hfill
 \begin{subfigure}{.24\textwidth}
  \centering
  \input{./figs/episodes_crash_episodes_Benchmarks.GRID_WORLD.tikz}%
  \caption{\GridWorld{}}
 \end{subfigure}
 \hfill
 \begin{subfigure}{.24\textwidth}
  \centering
  \input{./figs/episodes_crash_episodes_Benchmarks.CLIFFWALKING.tikz}%
  \caption{\CliffWalking{}}
 \end{subfigure}
 \hfill
 \begin{subfigure}{.24\textwidth}
  \centering
  \input{./figs/episodes_crash_episodes_Benchmarks.TAXI.tikz}%
  \caption{\Taxi{}}
 \end{subfigure}
 \hfill
 \begin{subfigure}{.25\textwidth}
  \centering
  \input{./figs/episodes_crash_episodes_Benchmarks.SELF_DRIVING_CAR.tikz}%
  \caption{\SelfDrivingCar{}}
 \end{subfigure}
 \hfill
 \begin{subfigure}{.24\textwidth}
  \centering
  \input{./figs/episodes_crash_episodes_Benchmarks.SIDEWALK.tikz}%
  \caption{\SideWalk{}}
 \end{subfigure}
 \hfill
 \begin{subfigure}{.24\textwidth}
  \centering
  \input{./figs/episodes_crash_episodes_Benchmarks.CAR_RACING.tikz}%
  \caption{\CarRacing{}}
 \end{subfigure}
 \begin{minipage}[t]{.03\textwidth} 
  \hfill
 \end{minipage}
 \begin{minipage}[t]{.21\textwidth}
  \centering
  \input{./figs/legend}%
 \end{minipage}
 \caption{The mean of the number of undesired training episodes and the number of training episodes before each testing phase}%
 \label{figure:crashes_vs_episodes}
\end{figure}

\cref{figure:crashes_vs_steps,figure:crashes_vs_episodes} show the mean of the reward of the number of undesired training episodes and the number of training steps and episodes, respectively.
In \cref{figure:crashes_vs_steps}, we observe that the number of training episodes with undesired behavior tends to be smaller when using \DynamicShielding{}.
This confirms the experiment results we observed and discussed in \cref{subsection:safety}.
In \cref{figure:crashes_vs_steps}, we also observe that the curve of \NoShield{} is steeper than \DynamicShielding{} in \CliffWalking{}, \Taxi{}, and \CarRacing{}.
This shows that when we train a controller without shielding, the learning agent keeps causing undesired behaviors while dynamic shielding prevents them at least partially.
In \cref{figure:crashes_vs_steps:taxi}, we observe that although the curve of \NoShield{} is steeper than \SafePadding{}, 
the curve of \SafePadding{} is still steeper than \DynamicShielding{}.
This shows that \SafePadding{} also prevents causing some undesired explorations, \DynamicShielding{} prevents more undesired explorations.
\cref{figure:crashes_vs_episodes} also shows a similar tendency.

\subsection{Safety of controllers learned with dynamic shielding}

\begin{figure}[tbp]
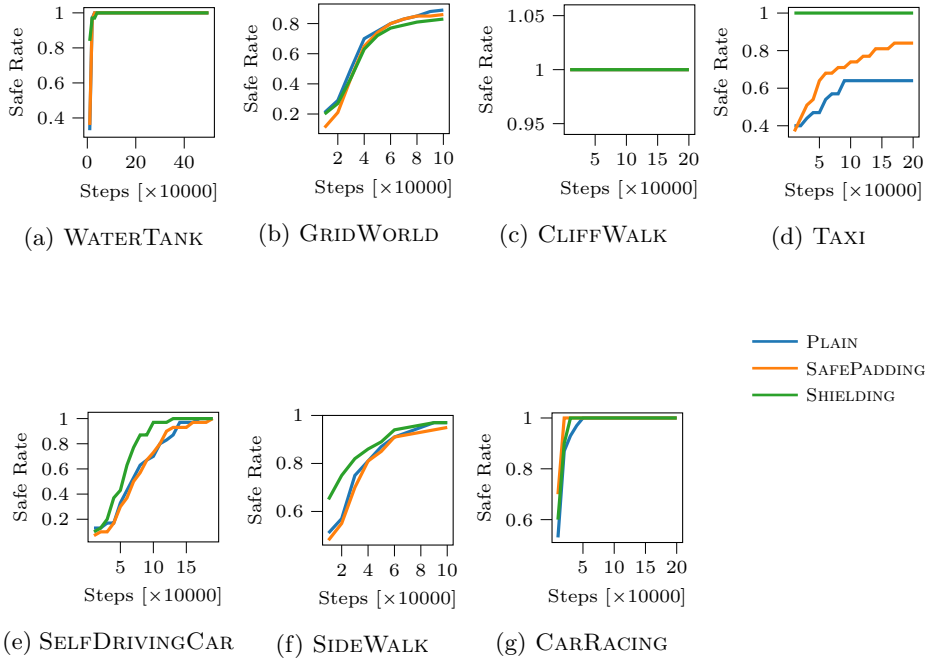

 \scriptsize
 \begin{subfigure}{.24\textwidth}
  \centering
  \input{./figs/steps_eval-best_safe_rate_Benchmarks.WATER_TANK.tikz}%
  \caption{\WaterTank{}}
 \end{subfigure}
 \hfill
 \begin{subfigure}{.24\textwidth}
  \centering
  \input{./figs/steps_eval-best_safe_rate_Benchmarks.GRID_WORLD.tikz}%
  \caption{\GridWorld{}}
 \end{subfigure}
 \hfill
 \begin{subfigure}{.24\textwidth}
  \centering
  \input{./figs/steps_eval-best_safe_rate_Benchmarks.CLIFFWALKING.tikz}%
  \caption{\CliffWalking{}}
 \end{subfigure}
 \hfill
 \begin{subfigure}{.24\textwidth}
  \centering
  \input{./figs/steps_eval-best_safe_rate_Benchmarks.TAXI.tikz}%
  \caption{\Taxi{}}
 \end{subfigure}
 \hfill
 \begin{subfigure}{.25\textwidth}
  \centering
  \input{./figs/steps_eval-best_safe_rate_Benchmarks.SELF_DRIVING_CAR.tikz}%
  \caption{\SelfDrivingCar{}}
 \end{subfigure}
 \hfill
 \begin{subfigure}{.24\textwidth}
  \centering
  \input{./figs/steps_eval-best_safe_rate_Benchmarks.SIDEWALK.tikz}%
  \caption{\SideWalk{}}
 \end{subfigure}
 \hfill
 \begin{subfigure}{.24\textwidth}
  \centering
  \input{./figs/steps_eval-best_safe_rate_Benchmarks.CAR_RACING.tikz}%
  \caption{\CarRacing{}}
 \end{subfigure}
 \begin{minipage}[t]{.03\textwidth} 
  \hfill
 \end{minipage}
 \begin{minipage}[t]{.21\textwidth}
  \centering
  \input{./figs/legend}%
 \end{minipage}
 \caption{The mean of the training steps before each testing phase and the safe rate of the best controller obtained by the testing phase.}%
 \label{figure:steps_vs_safe_rate}
\end{figure}
\begin{figure}[tbp]
 \scriptsize
 \begin{subfigure}{.24\textwidth}
  \centering
  \input{./figs/episodes_eval-best_safe_rate_Benchmarks.WATER_TANK.tikz}%
  \caption{\WaterTank{}}
 \end{subfigure}
 \hfill
 \begin{subfigure}{.24\textwidth}
  \centering
  \input{./figs/episodes_eval-best_safe_rate_Benchmarks.GRID_WORLD.tikz}%
  \caption{\GridWorld{}}
 \end{subfigure}
 \hfill
 \begin{subfigure}{.24\textwidth}
  \centering
  \input{./figs/episodes_eval-best_safe_rate_Benchmarks.CLIFFWALKING.tikz}%
  \caption{\CliffWalking{}}
 \end{subfigure}
 \hfill
 \begin{subfigure}{.24\textwidth}
  \centering
  \input{./figs/episodes_eval-best_safe_rate_Benchmarks.TAXI.tikz}%
  \caption{\Taxi{}}
 \end{subfigure}
 \hfill
 \begin{subfigure}{.25\textwidth}
  \centering
  \input{./figs/episodes_eval-best_safe_rate_Benchmarks.SELF_DRIVING_CAR.tikz}%
  \caption{\SelfDrivingCar{}}
 \end{subfigure}
 \hfill
 \begin{subfigure}{.24\textwidth}
  \centering
  \input{./figs/episodes_eval-best_safe_rate_Benchmarks.SIDEWALK.tikz}%
  \caption{\SideWalk{}}
 \end{subfigure}
 \hfill
 \begin{subfigure}{.24\textwidth}
  \centering
  \input{./figs/episodes_eval-best_safe_rate_Benchmarks.CAR_RACING.tikz}%
  \caption{\CarRacing{}}
 \end{subfigure}
 \begin{minipage}[t]{.03\textwidth} 
  \hfill
 \end{minipage}
 \begin{minipage}[t]{.21\textwidth}
  \centering
  \input{./figs/legend}%
 \end{minipage}
 \caption{The mean of the training episodes before each testing phase and the safe rate of the best controller obtained by the testing phase.}%
 \label{figure:episodes_vs_safe_rate}
\end{figure}
\begin{figure}[tbp]
 \scriptsize
 \begin{subfigure}{.24\textwidth}
  \centering
  \input{./figs/crash_episodes_eval-best_safe_rate_Benchmarks.WATER_TANK.tikz}%
  \caption{\WaterTank{}}
 \end{subfigure}
 \hfill
 \begin{subfigure}{.24\textwidth}
  \centering
  \input{./figs/crash_episodes_eval-best_safe_rate_Benchmarks.GRID_WORLD.tikz}%
  \caption{\GridWorld{}}
 \end{subfigure}
 \hfill
 \begin{subfigure}{.24\textwidth}
  \centering
  \input{./figs/crash_episodes_eval-best_safe_rate_Benchmarks.CLIFFWALKING.tikz}%
  \caption{\CliffWalking{}}
 \end{subfigure}
 \hfill
 \begin{subfigure}{.24\textwidth}
  \centering
  \input{./figs/crash_episodes_eval-best_safe_rate_Benchmarks.TAXI.tikz}%
  \caption{\Taxi{}}
 \end{subfigure}
 \hfill
 \begin{subfigure}{.25\textwidth}
  \centering
  \input{./figs/crash_episodes_eval-best_safe_rate_Benchmarks.SELF_DRIVING_CAR.tikz}%
  \caption{\SelfDrivingCar{}}
 \end{subfigure}
 \hfill
 \begin{subfigure}{.24\textwidth}
  \centering
  \input{./figs/crash_episodes_eval-best_safe_rate_Benchmarks.SIDEWALK.tikz}%
  \caption{\SideWalk{}}
 \end{subfigure}
 \hfill
 \begin{subfigure}{.24\textwidth}
  \centering
  \input{./figs/crash_episodes_eval-best_safe_rate_Benchmarks.CAR_RACING.tikz}%
  \caption{\CarRacing{}}
 \end{subfigure}
 \begin{minipage}[t]{.03\textwidth} 
  \hfill
 \end{minipage}
 \begin{minipage}[t]{.21\textwidth}
  \centering
  \input{./figs/legend}%
 \end{minipage}
 \caption{The mean of the training episodes with undesired behavior before each testing phase and the safe rate of the best controller obtained by the testing phase.}%
 \label{figure:crash_episodes_vs_safe_rate}
\end{figure}

\cref{figure:steps_vs_safe_rate,figure:episodes_vs_safe_rate,figure:crash_episodes_vs_safe_rate} show the mean safe rate of the best controller obtained by each testing phase and the  number of training steps and episodes and the number of episodes with undesired behavior by the testing phase, respectively.
In \cref{figure:steps_vs_safe_rate,figure:episodes_vs_safe_rate,figure:crash_episodes_vs_safe_rate}, we generally observe that in most of the benchmarks, the curve of \DynamicShielding{} is growing faster than of those of \NoShield{} and \SafePadding{}.
This shows that \DynamicShielding{} successfully prevented  undesired explorations.

\subsection{Other experiment results}

\begin{table*}[tbp]
 \centering
 \caption{The mean and the standard deviation of the number of the training episodes with undesired behavior out of 30 trials}%
 \label{table:detail_of_safety_violations}
 \small
 \begin{tabular}{lcccccccccccc}
  \toprule
  {} & \multicolumn{2}{c}{\NoShield{}} & \multicolumn{2}{c}{\SafePadding{}} & \multicolumn{2}{c}{\DynamicShielding{}} \\
  {} & mean & std & mean & std & mean & std \\
  \midrule
  \WaterTank{} & 1883.67 & 83.02 & 1892.40 & 79.93 & 177.13 & 28.04 \\
  \GridWorld{} & 6996.40 & 1510.04 & 7322.23 & 1399.39 & 5623.43 & 1256.44 \\
  \CliffWalking{} & 1493.20 & 1523.12 & 1528.67 & 1058.68 & 478.20 & 155.09 \\
  \Taxi{} & 8723.13 & 5593.19 & 2057.33 & 1406.02 & 37.77 & 19.79 \\
  \SelfDrivingCar{} & 6403.07 & 678.95 & 6454.60 & 895.09 & 5662.40 & 582.88 \\
  \SideWalk{} & 373.60 & 386.11 & 427.93 & 399.94 & 273.37 & 314.05 \\
  \CarRacing{} & 180.13 & 136.85 & 141.17 & 132.76 & 41.73 & 66.79 \\
  \bottomrule
 \end{tabular}
\end{table*}

\begin{table*}[tbp]
 \centering
 \caption{The mean and the standard deviation of the mean reward out of 30 trials}%
 \label{table:detail_of_success_rate}
 \small
 \begin{tabular}{lcccccccccccc}
  \toprule
  {} & \multicolumn{2}{c}{\NoShield{}} & \multicolumn{2}{c}{\SafePadding{}} & \multicolumn{2}{c}{\DynamicShielding{}} \\
  {} & mean & std & mean & std & mean & std \\
  \midrule
  \WaterTank{} & 918.89 & 9.50 & 919.81 & 8.15 & 921.81 & 3.46 \\
  \GridWorld{} & 0.37 & 0.55 & 0.46 & 0.54 & 0.07 & 0.59 \\
  \CliffWalking{} & -69.13 & 41.93 & -66.00 & 42.36 & -65.93 & 42.44 \\
  \Taxi{} & -147.61 & 222.32 & -139.62 & 224.60 & -92.93 & 26.90 \\
  \SelfDrivingCar{} & 28.83 & 3.04 & 28.86 & 2.18 & 29.81 & 0.55 \\
  \SideWalk{} & 0.93 & 0.44 & 0.90 & 0.50 & 0.67 & 0.54 \\
  \CarRacing{} & 375.53 & 419.98 & 509.25 & 450.89 & 622.07 & 346.90 \\
  \bottomrule
 \end{tabular}
\end{table*}

\begin{table*}[tbp]
 \centering
 \caption{The mean and the standard deviation of the execution time (in seconds) out of 30 trials}%
 \label{table:detail_of_execution_time}
 \small
 \begin{tabular}{lcccccccccccc}
  \toprule
  {} & \multicolumn{2}{c}{\NoShield{}} & \multicolumn{2}{c}{\SafePadding{}} & \multicolumn{2}{c}{\DynamicShielding{}} \\
  {} & mean & std & mean & std & mean & std \\
  \midrule
  \WaterTank{} & 1860.46 & 59.08 & 1947.09 & 80.11 & 6080.89 & 4055.24 \\
  \GridWorld{} & 177.18 & 13.07 & 1487.10 & 125.03 & 4548.70 & 1667.95 \\
  \CliffWalking{} & 355.18 & 11.82 & 365.54 & 13.97 & 839.06 & 148.45 \\
  \Taxi{} & 336.04 & 10.22 & 349.78 & 12.07 & 611.87 & 124.57 \\
  \SelfDrivingCar{} & 865.55 & 3.47 & 4919.13 & 168.00 & 10087.18 & 707.48 \\
  \SideWalk{} & 762.67 & 64.92 & 1734.66 & 146.83 & 6395.73 & 1460.23 \\
  \CarRacing{} & 7650.38 & 811.20 & 16694.63 & 576.96 & 12532.04 & 667.43 \\
  \bottomrule
 \end{tabular}
\end{table*}

\begin{table*}[tbp]
 \centering
 \caption{The difference of the mean execution time (in seconds) out of 30 trials.}%
 \label{table:difference_execution_time}
 \small
 \begin{tabular}{lcccccccccccc}
  \toprule
  {} & $\DynamicShielding{} - \NoShield{}$ & $\SafePadding{} - \NoShield{}$ \\
  \midrule
  \WaterTank{} & 4220.43 & 86.62 \\
  \GridWorld{} & 4371.53 & 1309.92 \\
  \CliffWalking{} & 483.88 & 10.37 \\
  \Taxi{} & 275.83 & 13.74 \\
  \SelfDrivingCar{} & 9221.62 & 4053.57 \\
  \SideWalk{} & 5633.06 & 971.99 \\
  \CarRacing{} & 4881.66 & 9044.25 \\
  \bottomrule
 \end{tabular}
\end{table*}

\cref{table:detail_of_safety_violations,table:detail_of_success_rate,,table:detail_of_execution_time} show the mean and the standard deviation of
\begin{oneenumeration}
 \item the number of the training episodes with undesired behavior,
 \item the success rate, and
 \item the total execution time (in seconds)
\end{oneenumeration}
out of 30 trials, respectively.
\cref{table:difference_execution_time} shows the mean of the difference of the total execution time between \DynamicShielding{} or \SafePadding{}, and \NoShield{}.
\fi
\end{document}